%% file: main.tex
\documentclass{article}
\usepackage{iclr2026_conference,times}

\usepackage[T1]{fontenc}
\usepackage[utf8]{inputenc}

\usepackage{amsmath, amssymb}
\usepackage{mathtools}
\usepackage{amsthm}
\newtheoremstyle{tight}{4pt}{4pt}{}{0pt}{\bfseries}{.}{0.5em}{}
\theoremstyle{tight}
\newtheorem{definition}{Definition}

\newtheorem{theorem}{Theorem}
\newtheorem{proposition}[theorem]{Proposition}
\newtheorem{corollary}[theorem]{Corollary}
\usepackage{hyperref}
\usepackage{url}
\usepackage{graphicx}
\usepackage{microtype}

\usepackage{booktabs}
\usepackage{enumitem}

\usepackage{algorithm}
\usepackage[noend]{algpseudocode}
\algrenewcommand\algorithmicrequire{\textbf{Input:}}
\algrenewcommand\algorithmicensure{\textbf{Output:}}
\algrenewcommand\alglinenumber[1]{\footnotesize#1:}

\usepackage{tikz}
\usetikzlibrary{
  calc,
  arrows.meta,
  decorations.pathmorphing,
  decorations.markings,
  patterns,
  positioning,
  shapes.geometric,
  quotes,
  graphs,
  fit,
  backgrounds
}

\usepackage{xcolor}
\definecolor{cInk}{HTML}{222222}
\definecolor{cFlat}{HTML}{7A7A7A}
\definecolor{cInput}{HTML}{0072B2}
\definecolor{cAssur}{HTML}{009E73}
\definecolor{cTrajectory}{HTML}{CC79A7}
\definecolor{cSecondary}{HTML}{E69F00}
\definecolor{cWarning}{HTML}{D55E00}
\colorlet{cGround}{cFlat}
\colorlet{cDriver}{cInput}
\colorlet{cDyad}{cAssur}
\colorlet{cJoint}{cInk}
\colorlet{cCurve}{cTrajectory}
\colorlet{cBox}{cFlat!8}
\colorlet{cBoxTeal}{cAssur!9}
\colorlet{cClassII}{cDyad}
\colorlet{cClassIII}{cSecondary}
\hypersetup{
  colorlinks=true,
  linkcolor=cInput,
  citecolor=cAssur,
  urlcolor=cInput
}

\tikzset{
  link/.style      = {line width=4pt, line cap=round, line join=round},
  ground link/.style = {link, color=cGround!70},
  driver link/.style = {link, color=cDriver},
  dyad link/.style   = {link, color=cDyad},
  class3 link/.style = {link, color=cClassIII},
  joint/.style     = {circle, draw=cJoint, fill=white,
                       inner sep=0pt, minimum size=8pt, line width=1pt},
  fixed joint/.style = {rectangle, draw=cJoint, fill=cGround!20,
                        inner sep=0pt, minimum size=8pt, line width=1pt},
  interface joint/.style = {joint, fill=white},
  internal joint/.style = {joint, fill=cAssur, draw=cAssur},
  ground/.style    = {
    postaction={
      pattern=north east lines,
      pattern color=cGround!50
    }
  },
  group box/.style = {
    draw, rounded corners=8pt, dashed, thick,
    inner sep=8pt
  },
  node label/.style  = {font=\small\bfseries},
  link label/.style  = {font=\small\itshape, midway},
  annotation/.style  = {font=\footnotesize, text=black!70},
  arr/.style         = {-{Stealth[length=5pt]}, thick},
  graph node/.style  = {
    circle, draw=cJoint, fill=white,
    minimum size=24pt, font=\small,
    line width=1pt
  },
  scc box/.style = {
    draw, rounded corners=6pt, dashed,
    inner sep=6pt
  },
  rigid body/.style      = {fill=#1!15, draw=#1, rounded corners=5pt,
                            line width=0.8pt},
  rigid body/.default    = cDyad,
  prismatic joint/.style = {draw=cJoint, fill=white, rectangle,
                            minimum size=9pt, inner sep=0pt, line width=1pt},
  dag node/.style        = {draw=#1, fill=#1!12, rounded corners=6pt,
                            inner sep=6pt, font=\small, line width=1pt},
  dag node/.default      = cJoint,
  pipeline box/.style    = {draw, rounded corners=6pt, fill=cBox,
                            inner sep=7pt, font=\small, align=center},
}

\newcommand{\DOF}{\mathcal{F}}

\title{Efficient Learning of Mechanism Designs and Compositionality}
\author{
Jo\~{a}o Pedro Gandarela\textsuperscript{1,2} \quad
Thiago Rios\textsuperscript{3} \quad
Stefan Menzel\textsuperscript{3} \quad
Andr\'{e} Freitas\textsuperscript{1,4,5} \\
\textsuperscript{1}Idiap Research Institute, Switzerland \quad
\textsuperscript{2}\'{E}cole Polytechnique F\'{e}d\'{e}rale de Lausanne (EPFL), Switzerland \\
\textsuperscript{3}Honda Research Institute Europe, Germany \\
\textsuperscript{4}Department of Computer Science, University of Manchester, UK \\
\textsuperscript{5}National Biomarker Centre, CRUK-MI, University of Manchester, UK \\
\texttt{firstname.lastname@idiap.ch} \quad
\texttt{firstname.lastname@honda-ri.de}
}

\date{}

\iclrfinalcopy

\begin{document}
\maketitle

\input{sections/00_abstract}
\input{figures/hero_figure}

\input{sections/01_introduction}

\input{sections/02_background}
\input{sections/04_knowledge_base}
\input{sections/05_theory}
\input{sections/06_benchmarks}
\input{sections/03_related_work}
\input{sections/08_conclusion}

\bibliographystyle{iclr2026_conference}
\bibliography{references}

\appendix
\input{sections/llm_disclosure}
\input{appendices/A_background}
\input{appendices/H_benchmark_tables}

\input{appendices/J_theory}
\input{appendices/K_kb_census}

\end{document}

%% file: sections/00_abstract.tex
\begin{abstract}
\noindent
Learning mechanism behaviour requires a representation of how connected parts jointly constrain motion. We study whether exposing reusable constraint subassemblies improves prediction from limited data. Assur-KB represents planar mechanisms through irreducible modules, boundary interfaces, and solve dependencies while retaining source geometry. Our theoretical results characterize transfer through module interfaces and establish physical limits arising from missing interface information and accumulated sensitivity. To test the representation claim directly, the experimental design compares matched mechanism-curve pairs across training budgets within the same multilayer-perceptron family. This controlled comparison establishes the utility of the complete representation in the evaluated learning setting. Assur-derived features provide a geometric-mean $1.8\times$ advantage in sample efficiency over flattened adjacency and position inputs across budgets of 50-10,000 mechanisms. Separate experiments establish invariance under tested joint relabelings, higher valid-generation rates under dyadic composition, and a coarser structural index over 10.02 million mechanisms. In summary, the theory and experiments identify how physical constraint structure supports limited-data prediction and which interface conditions govern compositional reuse.
\end{abstract}

%% file: figures/hero_figure.tex
\begin{figure}[h]
  \centering
  \includegraphics[width=\linewidth]{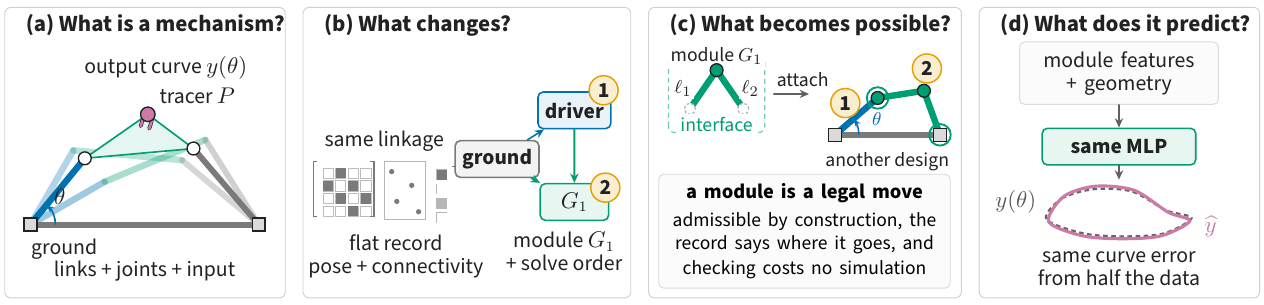}
  \caption{\textbf{Same mechanism and motion, different computational structure made explicit.} \textbf{(a)} A driver angle moves the links and a tracer point $P$ draws $y(\theta)$.  \textbf{(b)} The solve graph exposes an irreducible module $G_1$, its interface, and the order the local solves run in. \textbf{(c)} That module is then a move an agent can make, stored with the two bar lengths and the interface it needs, landing on any pair of already-solved joints, admissible by construction, and extending the solve
  order by one step. \textbf{(d)} The same record is what a predictor reads, so with the multilayer perceptron held fixed, module features and geometry reach a given error on the curve $y(\theta)$ from close to half the training data (\S\ref{sec:exp-regression}).}
  \label{fig:hero-solve-graph}
\end{figure}

%% file: sections/01_introduction.tex
\section{Introduction}
\label{sec:introduction}

Learning the behaviour of physical systems raises a fundamental question about representation: which parts of a system provide useful units for learning and reuse? Compositional representations offer a way to share knowledge across systems by making their components and relationships explicit. Their effectiveness, however, depends on identifying units that capture the structure relevant to the behaviour being predicted. This question connects representation learning with relational inductive biases and modular learning, where the organization of a problem shapes how knowledge can transfer across examples \citep{battaglia2018relational,alet2018modular}.

Mechanical linkages provide a concrete setting for studying this question. A linkage consists of connected rigid links and joints whose motion transforms a driving input into an output trajectory. Learning this relationship is difficult because the components satisfy coupled geometric constraints, where the movement of one joint depends on the configurations of other connected parts. Moreover, mechanisms with the same connectivity can produce different trajectories as their geometry or assembly configuration changes. Across mechanisms, the combinations of topology, geometry, and attachment patterns create a large space of behaviours that training examples can only partially cover.

The representational challenge is therefore to expose reusable structure at an appropriate level of granularity. Joint-level graphs describe elementary connections, but leave implicit the subassemblies whose constraints must be solved together and the dependencies between those subassemblies. A learner must recover this organization insofar as it is useful for prediction. Making it explicit could help the learner recognize recurring constraint patterns across mechanisms, reducing the amount of data needed to learn the relationship between structure and motion.

Planar mechanism theory provides a principled candidate for these representational units. Assur groups are irreducible subassemblies whose internal configurations can be determined, subject to geometric feasibility and assembly choices, once their boundary configurations are supplied \citep{ServatiusShaiWhiteley2010}. Each group thus defines a conditional constraint problem with explicit interfaces. Connecting these groups according to their dependencies produces a \textit{directed acyclic solve graph}, which specifies how local constraint solutions combine to determine the mechanism's configuration \citep{ShaiSljokaWhiteley2013}. Repeated across the driving input, this computation determines its motion.

We investigate whether exposing this \textit{compositional structure} improves learning of mechanism behaviour. We represent mechanisms through their Assur groups, interfaces, and solve dependencies, while retaining the underlying geometry and assembly information. Our hypothesis is that this organization supplies a useful inductive bias by making recurring constraint structure accessible across examples. We instantiate the representation at corpus scale in Assur-KB and evaluate its value for predicting output trajectories, alongside its usefulness for structural composition and indexing. To our knowledge, this is the first large-scale evaluation of Assur-based representations for learning mechanism behaviour from limited data.

Our experiments provide evidence for improved learning when data is limited. Using the same multilayer perceptron, Assur-derived features reach comparable curve-prediction error with roughly half the training examples required by a flattened adjacency-and-position representation over the evaluated low-data range (Figure~\ref{fig:hero-solve-graph}d). More broadly, this approach suggests that physical constraint structure can guide mechanism representation design, potentially reducing simulation requirements for predictive modelling and supporting modular reuse across related physical systems.

%% file: sections/02_background.tex
\section{From Physical Parts to Constraint Modules}
\label{sec:background}

\paragraph{The prediction problem.}
A planar linkage consists of rigid links connected by revolute joints (pins) or prismatic joints (sliders). Ground fixes part of the linkage, while a driver supplies an input, such as a crank angle $\theta$. Solving the geometric constraints determines the other joint positions and a tracer trajectory $y(\theta)$. We study learning this kinematic behaviour from mechanism examples. The regression experiment predicts the shape of the resulting curve after normalization and arc-length resampling.

A common source representation is $(A,x_0,n)$, consisting of joint adjacency, an initial configuration, and a fixed-joint mask. Together with the designated driver and tracer, it supplies the information needed for the source simulation. A link-incidence representation additionally records which joint sets belong to the same rigid body. Both forms preserve physical connectivity, but neither explicitly identifies which constraints form a local solve problem.

\paragraph{A concrete module.}
Consider an unknown joint $q$ connected by two bars of lengths $\ell_1,\ell_2$ to known joints $p_1,p_2$ (Figure~\ref{fig:dyad-peeling}a). Its position satisfies
\begin{equation}
 \|q-p_1\|^2=\ell_1^2,\qquad \|q-p_2\|^2=\ell_2^2.
 \label{eq:module-dyad}
\end{equation}
The bars must be jointly solved as for intersecting, non-tangent circles there are two solutions, while an assembly branch selects one. The pair of bars is therefore a useful computational unit, called a \emph{dyad}. Reusing it means reusing a conditional constraint relationship, where different boundary positions and lengths generally produce different motions.

\begin{definition}[Assur group]
\label{def:assur-group}
An Assur group is a subassembly that, with its boundary joints pinned, is generically rigid and independent, and has no proper subassembly with these properties \citep{ServatiusShaiWhiteley2010}.
\end{definition}

The boundary joints are the module's inputs, where its internal joints are the unknowns solved together. Irreducibility specifies the granularity when a group cannot be split into smaller sequentially solvable units under the same pinning. The dyad is the simplest example, while larger groups contain several coupled internal joints.

\paragraph{From modules to a solve graph.}
Draw an edge $G_i\to G_j$ whenever group $G_j$ uses a boundary joint determined by $G_i$. For a valid Assur decomposition this gives a directed acyclic graph (DAG), and a topological traversal gives a dependency-respecting solve order \citep{ShaiSljokaWhiteley2013}. This graph exposes the organization of the constraint computation while preserving the mechanism's geometry. It provides the compositional structure evaluated in this paper.

\paragraph{Structure and geometric feasibility.}
Structural mobility and successful motion are distinct concepts. For $L$ links including ground and $J_1$ one-degree-of-freedom lower pairs, the nominal Chebychev--Gr\"ubler--Kutzbach count is
\begin{equation}
 \DOF=3(L-1)-2J_1.
 \label{eq:mobility}
\end{equation}
This count assumes independent constraints, where special geometries or redundant constraints can change the actual mobility \citep{Mueller2009GenericMobility}. An Assur decomposition describes generic constraint structure. Assembly at particular lengths and continuation through a full input cycle are geometric questions: e.g. the dyad in Eq.~\ref{eq:module-dyad} has no real solution when its circles do not intersect. Appendix~\ref{app:background} develops these ideas and the \textit{local solves} in detail.

%% file: sections/04_knowledge_base.tex
\section{A Compositional Representation of Mechanisms}
\label{sec:knowledge-base}

Each record lifts one joint graph into a \textit{solve graph} whose nodes are irreducible groups, whose edges are the dependencies between them, and whose statuses record what was verified and what was left open. The construction has two stages, a proposal stage belonging to the source format and an acceptance stage both sources share, which is what makes a structural fingerprint mean the same thing in either (Figure~\ref{fig:kb-pipeline}). Appendix~\ref{app:kb-census} provides a descriptive account of the build.

Write a source mechanism as $R=(J,\mathcal{L},\tau,\Gamma,(u,v),x_0)$, with joints $J$, rigid links $\mathcal{L}\subseteq 2^{J}$ held as joint sets so that a binary bar is a link $\ell$ with $|\ell|=2$, a map $\tau:J\to\{\mathrm{R},\mathrm{P}\}$ marking each joint revolute or prismatic, grounded joints $\Gamma$, a driver turning about fixed pivot $u$ and carrying joint $v$, and an optional pose $x_0$. Joint adjacency is the case $\mathcal{L}=\{\{i,j\}:A_{ij}=1\}$ with $\tau\equiv\mathrm{R}$, and link incidence reads $\mathcal{L}$ off the rows of an incidence matrix $B$, one row per link, with $\tau$ from the slider specification, so one schema carries both corpora. The build turns $R$ into a sequence of groups $\mathcal{G}=(G_0,\dots,G_m)$ headed by the ground and driver $G_0$, where $G_k=(I_k,\partial_k,\mathcal{L}_k,T_k,\beta_k,s_k)$ for $k\ge1$ holds the joints group $k$ solves, its boundary, the links it owns, the tracer points it carries, its assembly branch, and its status.

\begin{definition}[Valid decomposition]
\label{def:valid-decomposition}
Let $I_0=\Gamma\cup\{v\}$ and let $K_k=I_0\cup\dots\cup I_{k-1}$ be the joints known before step $k$. The sequence $\mathcal{G}$ is a \emph{valid decomposition} of $R$ when the blocks $I_k\cup T_k$ partition $J$, the $\mathcal{L}_k$ partition $\mathcal{L}$, every boundary is already known, $\partial_k\subseteq K_k$, and every $G_k$ with $k\ge1$ is an Assur group on $(I_k,\partial_k,\mathcal{L}_k)$ in the sense of Definition~\ref{def:assur-group}.
\end{definition}

Ownership then induces the edges $E=\{(j,k):\partial_k\cap I_j\neq\emptyset\}$ between groups $j$ and $k$, which precedence makes acyclic, so every topological order $\sigma$ of $(\mathcal{G},E)$ is a solve order. Because the $\mathcal{L}_k$ partition $\mathcal{L}$ while $\Gamma$, $(u,v)$ and $x_0$ are carried through, $R$ is recoverable from its decomposition.

\begin{figure}[t]
 \centering
 \includegraphics[width=0.97\linewidth]{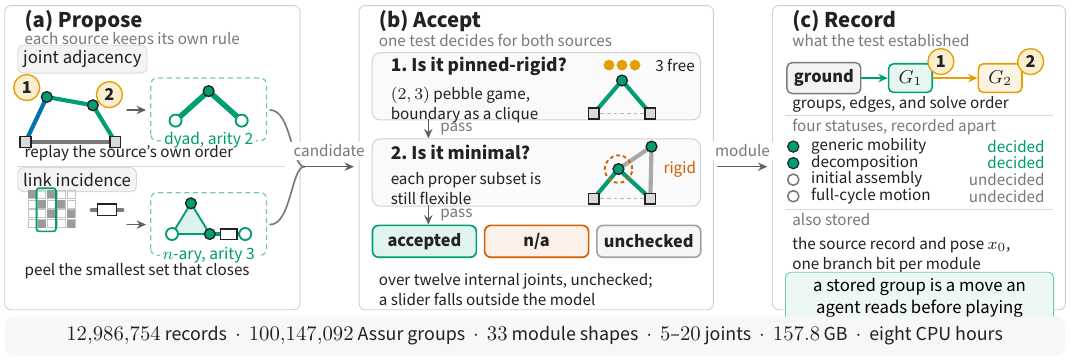}
 \caption{\textbf{Source formats(2x), acceptance test (1x) and record (1x).} \textbf{(a)} Each corpus proposes modules through the oracle it already ships, so joint adjacency yields the classical dyad and link incidence reaches $n$-ary bodies, prismatic pairs, and boundary arity above two. \textbf{(b)} Both proposals then meet the same test (Eq.~\ref{eq:accept}), a module counting as pinned-rigid exactly when three pebbles remain free and as minimal exactly when every proper subset of its internal joints is still flexible. \textbf{(c)} The record keeps the solve graph beside four statuses, so a structural pass never reads as motion.}
 \label{fig:kb-pipeline}
\end{figure}

\input{figures/build_algorithms}

\begin{figure}[t]
 \centering
 \includegraphics[width=0.9\linewidth]{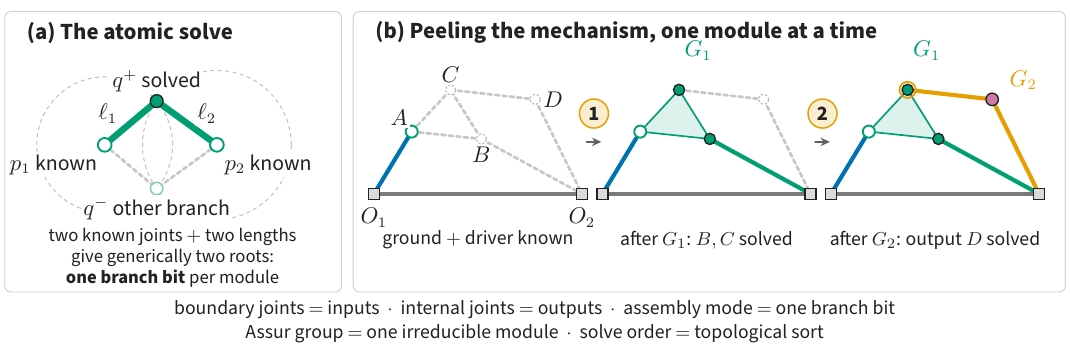}
 \caption{\textbf{Peeling proposes one module at a time against the part already solved.} \textbf{(a)} Two known joints and two lengths place the unknown joint at a circle--circle intersection with generically two real roots, so a mechanism with $m$ dyads stores $m$ branch bits. \textbf{(b)} Two steps on a six-bar, where the ternary link $ABC$ and the bar from $B$ to ground form one module because the ternary body enters once and fixes both $B$ and $C$, and the remaining dyad closes at $D$.}
 \label{fig:dyad-peeling}
\end{figure}

\subsection{Modules, interfaces, and dependencies}
\label{sec:extraction-method}

A mechanism corpus already entails a routine that orders the solves its own simulator performs, so the build reuses that routine as a proposal oracle and leaves the peeling rule with the source that defined it. \textsc{Peel} (Algorithm~\ref{alg:build}a) is the loop both sources run, growing a solved set from ground and the driven joint and reading each proposed set of links as internal, where it is unknown and, boundary, where it is known. Degree-one tracer points are held out, since a link carries them rigidly, and each is reattached to the group owning its link, a step only link incidence exercises. An oracle with no next set reports a failure, so no module is inferred from a neighbour count, though no record in the build takes that path. What does occur is a peel that resolves no new joint, as in the $153{,}124$ link-incidence records whose source incidence leaves every structural joint on the ground link, each keeping its curve and its ground group with no module below it.

\paragraph{Joint adjacency.} For LINKS~\citep{HeyraniNobari2022Links} (Appendix~\ref{app:links-dataset}) the oracle replays the corpus routine and returns, at each step, the two bars joining one unsolved joint to two already-known neighbours (Figure~\ref{fig:dyad-peeling}a). Every record of this form is ordered this way, and all of them yield the same module shape, the classical dyad.

\paragraph{Link incidence.} For the Purwar Lab corpus~\citep{purwarlab2024dataset} (Appendix~\ref{app:purwarlab-dataset}), where each rigid link retains its complete joint set, prismatic coupling is folded into the incidence first, which makes a slider an ordinary structural coupling and keeps its sliding joint from stalling the reduction at degree one. The oracle then returns the smallest set of two or more unsolved links whose Chebychev--Gr\"ubler--Kutzbach mobility against the solved part is non-positive, widening the set when none closes. This is the form that reaches $n$-ary bodies, prismatic pairs, and boundary arity above two.

\subsection{Structural tests and geometric information}
\label{sec:acceptance}

A proposal is a candidate, and \textsc{Accept} (Algorithm~\ref{alg:build}b) decides what it may be recorded as. Pinned rigidity is decided with the $(2,3)$ \textit{pebble game}~\citep{JacobsHendrickson1997,10.1115/DETC2011-48340}, exact for generic bar-joint frameworks and near-linear in the number of constraints (in essence, two freedom tokens per joint, one token claimed by each independent bar, and the module is rigid once only the plane's three are left; Appendix~\ref{app:pebble-game} describes the game from scratch). The boundary enters as a complete graph on the pinned joints, which keeps the plane's three residual freedoms explicit, and each link enters as one clique-expanded rigid body reduced to its independent constraints, so an $n$-ary link contributes once. Writing $p$ for the free pebbles that remain on module $k$ and $\mathcal{L}_k|_{I'}$ for the links meeting $I'$ and contained in $I'\cup\partial_k$,
\begin{equation}
  \mathsf{Rigid}(I_k,\partial_k,\mathcal{L}_k) \iff p=3,
  \qquad
  \mathsf{Minimal} \iff \neg\,\mathsf{Rigid}(I',\partial_k,\mathcal{L}_k|_{I'})
  \ \text{ for every } \emptyset\neq I'\subsetneq I_k ,
  \label{eq:accept}
\end{equation}
and the module is accepted as Assur when both hold, with the surplus $p-3$ stored as its residual freedoms.

Neither conjunct is spare, and a mobility count supplies neither, because the pinned count $|\mathcal{L}_k|=2|I_k|$ is necessary for generic isostaticity and blind to minimality, since two stacked dyads satisfy it, while the first dyad alone is already pinned-rigid. Acceptance therefore rests on the pebble game. Minimality is exponential in $|I_k|$, which is affordable because Assur groups are small, and beyond twelve internal joints it is recorded as \emph{unchecked}; no module in the build reaches that size. A module touching a prismatic pair falls outside the model, since a slider supplies no length for a bar-joint test to read.

\subsection{What the record stores}
\label{sec:reconstructability}

Alongside $R$ and the groups, each record names the source chunk it came from. Each physical link is given to the module resolving its latest joint, which keeps a ternary or quaternary body distinguishable from several binary bars, and each dyad stores the branch $\beta_k$ its source pose realizes, since a mechanism with $m$ dyads has $2^m$ assembly modes and one of them reproduces $x_0$, so the stored bit replays a record without a search. Modules are keyed by boundary arity $|\partial_k|$, which agrees with classical Assur order only under a one-to-one boundary-pair convention, so the schema reports arity and leaves classical order uninferred. The status $s_k$ keeps the four outcomes of Figure~\ref{fig:kb-pipeline}c apart, each passed, failed, or not evaluated.

Records are built in independent index partitions, where geometry is defined per record, so where a source supplies no coordinate array the design frame is recovered and accepted only once the traced point maps through it onto the stored curve, and $94.6\%$ of those records end up with a validated frame. Every module the model covers passes the structural test, $286{,}467$ fall outside it and are stored as undecided, and initial assembly and full-cycle motion are undecided build-wide, since a decomposition never sees the geometry that would settle them (Appendix~\ref{app:kb-census}).

\subsection{Using the representation for learning}
\label{sec:reconstructability}
\label{sec:agent-interface}

The learning experiment converts the decomposition into a fixed feature vector containing summaries of the modules, their interface structure, the dependency DAG, and available branch and geometric information. A three-layer MLP predicts the normalized output curve from this vector. The baseline uses the same network with flattened adjacency and position inputs. A fixed-width map is appropriate because both inputs are fixed-length and the target is resampled to a fixed number of points. Using the same network also avoids introducing relational biases from graph networks or transformers, which would confound the effect of the representation. The experiment therefore isolates the value of exposing constraint structure in the input.

The same representation supports two structural operations. A dyad can be attached to a compatible pair of existing boundary joints to construct a new candidate, after which its geometry is evaluated. A coarse fingerprint of the module counts and dependency graph can retrieve mechanisms with related constraint structure. These uses motivate the sampling and indexing experiments, while curve prediction remains the central learning test.

The combined build contains 12,986,754 mechanism records. Indexing is evaluated on 10,020,000 records, and the learning study uses 88,910 matched mechanism-curve pairs. 



%% file: figures/build_algorithms.tex
\begin{algorithm}[t]
\caption{\textbf{(a)} The peeling loop, shared by both sources and specialized
only through its proposal oracle $\Pi$, which returns a set $S$ of links or
$\bot$. \textbf{(b)} The test scoring every proposed module, where $F$ is the
pinned frame the pebble game runs on and $r$ the rigidity outcome.}
\label{alg:build}
\scriptsize\setlength{\baselineskip}{9pt}
\begin{minipage}[t]{0.455\linewidth}
\begin{algorithmic}[1]
\Procedure{Peel}{$R$, $\Pi$}
  \State $K\gets I_0=\Gamma\cup\{v\}$;\; $k\gets0$
  \State $U\gets\{j\in J:\deg j\ge2\}\setminus K$
         \Comment{tracers held out}
  \While{$U\neq\emptyset$}
    \State $S\gets\Pi(K,U,\mathcal{L})$;\;
           \textbf{if} $S=\bot$ \textbf{then fail}
    \State $k\gets k+1$;\; $\mathcal{L}_k\gets S$
    \State $I_k\gets\textstyle\bigcup S\setminus K$;\;
           $\partial_k\gets\textstyle\bigcup S\cap K$
    \State $\beta_k\gets$ branch of $G_k$ realized by $x_0$
    \State $K\gets K\cup I_k$;\; $U\gets U\setminus I_k$
  \EndWhile
  \State carry tracers $T_k$;\; give each $\ell$ to its latest $k$
  \State $s_k\gets\textsc{Accept}(I_k,\partial_k,\mathcal{L}_k)$ for every $k$
  \State \Return $\mathcal{G},\,E=\{(j,k):\partial_k\cap I_j\neq\emptyset\},\,\sigma$
\EndProcedure
\end{algorithmic}
\end{minipage}\hfill
\begin{minipage}[t]{0.515\linewidth}
\begin{algorithmic}[1]
\Procedure{Accept}{$I_k,\partial_k,\mathcal{L}_k$}
  \State \textbf{if} $\tau(j)=\mathrm{P}$ for some $j\in I_k\cup\partial_k$
         \textbf{then return} \textsc{n/a}
  \State $F\gets$ complete graph on $\partial_k$ \Comment{ground}
  \State $F\gets F\cup\{$clique on $\ell\cap(I_k\cup\partial_k)\}_{\ell\in\mathcal{L}_k}$
  \State $p\gets$ free pebbles of the $(2,3)$ game on $F$
  \State $r\gets[\,p=3\,]$ \Comment{Eq.~\ref{eq:accept}, left}
  \State \textbf{if} $|I_k|>12$ \textbf{then return} $(r,\text{unchecked},p\!-\!3)$
  \For{$\emptyset\neq I'\subsetneq I_k$}
    \State \textbf{if} $\mathsf{Rigid}(I',\partial_k,\mathcal{L}_k|_{I'})$
           \textbf{then return} $(r,\textsc{false},p\!-\!3)$
  \EndFor
  \State \Return $(r,\textsc{true},p-3)$
\EndProcedure
\end{algorithmic}
\end{minipage}
\end{algorithm}

%% file: sections/05_theory.tex
\section{When Constraint Modules Support Prediction}
\label{sec:theory}

An Assur group exposes a recurring conditional problem: given its boundary configuration, geometry, and assembly branch, determine its internal configuration. Reusing a learned solution in a new mechanism depends on three questions: are the local inputs familiar, how do local errors propagate, and which distinctions must the representation retain? Appendix~\ref{app:theory} develops these questions through bounds and explicit mechanism constructions.

\paragraph{Local reuse and global error.}
A new assembly can present familiar local prediction problems. Consider conditional predictors shared by module type $t$ and composed along a solve DAG. Let $\epsilon_t$ denote their population RMS error under training input distribution $\mu_t$, and $\nu_v$ the true input distribution at occurrence $v$ in the new assembly. Under the domain and Lipschitz assumptions of Theorem~\ref{thm:interface-transport},
\begin{equation}
 \|\widehat y-y\|_{L^2}
 \leq \sum_v w_v\!\left[\epsilon_{t(v)}+
 K_{t(v)}W_2(\nu_v,\mu_{t(v)})\right].
 \label{eq:main-transfer}
\end{equation}
The bracket combines error on familiar inputs with the cost of changing local conditions: $W_2$ measures input-distribution displacement and $K_t$ bounds residual variation. The weight $w_v$ measures how strongly that error can affect the output along downstream paths. Thus local coverage and controlled propagation support reuse across assemblies. This bound concerns composed predictors; our experiments test the representational benefit within a global MLP.

\paragraph{Why small modules can still be difficult.}
Successive modules can direct motion into directions that amplify it. Two unit-bar dyad chains with identical dependency graphs and local Jacobian singular values can consequently have endpoint sensitivities $1$ and $2^m$ to the same unit-speed crank (Theorem~\ref{thm:theory-dyad-sensitivity-separation}). At tangencies, a different effect appears: each solve can take a square root of the incoming displacement, reducing configuration regularity to exponent $2^{-m}$ after $m$ stages (Theorem~\ref{thm:theory-singular-depth}). These constructions explain why interface geometry matters alongside modular structure. The appendix connects trajectory-error bounds to normalized curve readouts through explicit normalization and sampling conditions.

\paragraph{What the representation must preserve.}
Even perfect local prediction cannot recover a distinction lost in its inputs. Exchanging which boundary port receives each dyad length changes output shape, including after 128-point arc-length sampling (Propositions~\ref{prop:port-information}--\ref{prop:port-information-sampling}). More generally, continuous features need enough dimensions to retain independently observable variations (Theorem~\ref{thm:observable-dimension}). Finally, Assur component size and undirected graph width do not bound one another (Theorem~\ref{thm:assur-treewidth-separation}): conditional solve structure provides a distinct description of complexity. Together, these results identify what useful modular representations expose: complete interfaces, recurring local problems, and their influence on the output.

%% file: sections/06_benchmarks.tex
\section{Experimental Evaluation}
\label{sec:benchmarks}


We first test the central hypothesis, i.e. \textit{does making constraint structure explicit improve prediction from limited data?} Three further experiments separate that benefit from the representation's other properties, with the full tables in Appendix~\ref{app:benchmark-tables}.

\subsection{Learning output trajectories with fewer examples}
\label{sec:exp-regression}

We test whether features organized by constraint structure help a learner where data is scarce, on the 88,910 matched mechanism-curve pairs spanning $N=8,\ldots,20$, none of which failed to decompose. Both representations feed the same three-layer MLP. The target is the output joint's normalized coupler curve, arc-length resampled to 128 points and scored by symmetric Chamfer distance, because a scalar summary would let a model match perimeter without matching shape. Ten joint-count-stratified folds give paired evaluation units, unconfounding mechanism size from fold.

The gain reads most directly on the data axis, where at 500 training mechanisms the Assur model reaches a Chamfer error of $0.1276$ that the adjacency baseline does not reach until roughly 980 examples, and that ratio stays between $1.6\times$ and $2.0\times$ for every budget from 50 through 10,000, with a geometric mean of $1.8\times$ (Figure~\ref{fig:exp3-learning-curve}). On the error axis the same effect is an $11\%$ reduction at 50 examples. The one-sided Nadeau-Bengio-corrected paired test over the ten folds falls below $0.05$ at every size from 2,000 through 20,000, with $p$ between $0.0078$ and $7.4\times10^{-5}$.

The advantage then closes, since at 40,000 examples adjacency is lower by $0.0004$, inside the fold-to-fold spread, and at the full 72,013-example pool the reversal resolves with adjacency lower by $0.0014$ at $6.6\times10^{-6}$. The gain therefore belongs to the low-data regime and to the adjacency-plus-position baseline.

\begin{figure}[t]
  \centering
  \includegraphics[width=0.65\linewidth]{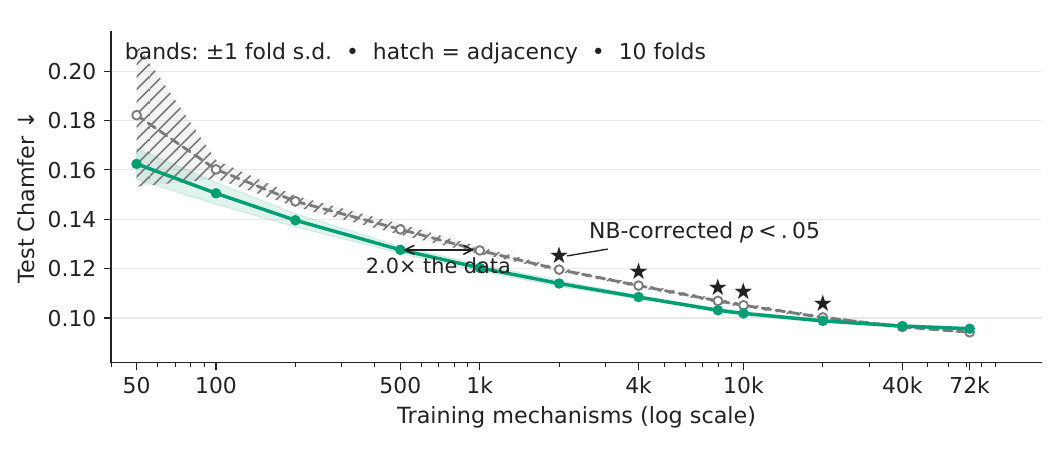}
  \caption{\textbf{Solve-graph features are worth about a factor of two in training data.} Mean symmetric Chamfer distance over ten stratified fold units, with the arrow reading the gain on the data axis and stars the sizes significant under the corrected paired test.}
  \label{fig:exp3-learning-curve}
\end{figure}

\subsection{Constructive sampling improves validity within the tested family}
\label{sec:exp-validity}

The \emph{compositional sampler} attaches every new joint as a two-boundary dyad. The \emph{matched adjacency sampler} draws its edges at random at the same expected edge count and forces the same motor edge, holding density fixed so that only the construction rule varies. A theoretical branch evaluates $N=4,\ldots,8$ at $\bar E_N=2N-3$ and an empirical branch $N=8,\ldots,12$ at statistics from LINKS itself. The oracle scores connectedness, $\DOF=1$ by Eq.~\ref{eq:mobility}, a solve path, and geometric simulation separately, and overall validity is their conjunction over 50,000 samples per cell.

\paragraph{Composition has higher overall validity at \emph{every} evaluated size.} In the theoretical branch its rate is 1.0 while the adjacency rate falls from 0.405 at $N=4$ to 0.035 at $N=8$. In the empirical branch composition falls from 0.284 at $N=8$ to 0.121 at $N=12$ and stays above adjacency sampling, whose rounded rates are at most 0.001. Over pairs where both are nonzero the geometric mean of the per-pair ratio is 54.5 (Figure~\ref{fig:structural-evidence}, left).

\paragraph{Composition reaches everything adjacency sampling reaches.} A 10,000-sample-per-cell coverage control rebuilds by composition every valid mechanism the matched adjacency sampler produced, succeeding in 100\% of cases at every branch and size, because the oracle's solve-path condition already establishes the dyadic stack composition assembles.

\subsection{Grouping mechanisms by coarse constraint structure}
\label{sec:exp-index}

Streaming 10,020,000 mechanisms, we compute for each a Weisfeiler-Lehman hash of the bare joint graph and a solve-graph fingerprint formed from group-order counts and the order-labelled condensation DAG. These yield 1,173,078 distinct hashes against 908,186 fingerprints, a $1.29\times$ coarsening averaging 11.03 mechanisms per fingerprint (Figure~\ref{fig:structural-evidence}b), because mechanisms a joint-indexed hash separates are pulled together when they decompose the same way.

\begin{figure}[t]
  \centering
  \includegraphics[width=0.65\linewidth]{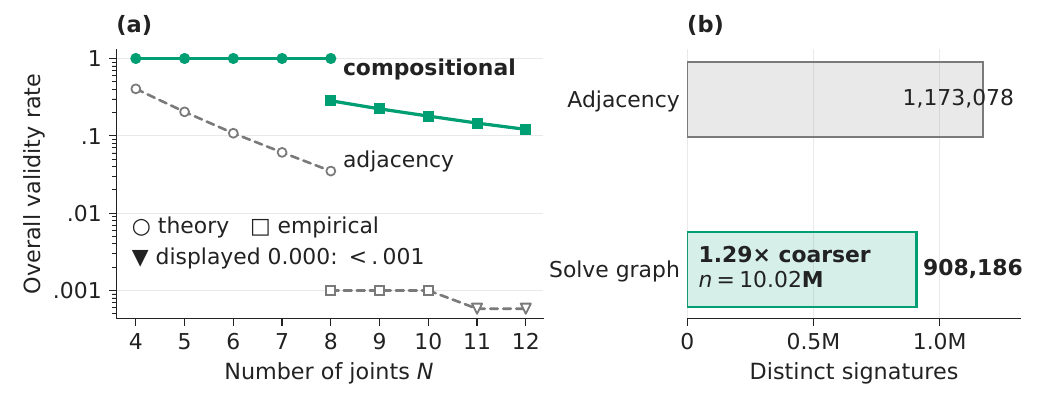}
  \caption{\textbf{Solve graphs expose useful structure before any learned model is introduced.} \textbf{(a)} Overall validity for compositional and matched adjacency sampling. \textbf{(b)} 10.02 million mechanisms yield fewer solve-graph fingerprints than bare adjacency hashes.}
  \label{fig:structural-evidence}
\end{figure}

\paragraph{Which feature blocks earn their place.}
Leave-one-out importance at $n=5{,}000$ over the same ten folds does not flatter the default vector, since the condensation-DAG block is the only clear contributor, the branch block the only one that costs anything, and the remaining six sit inside the fold-to-fold spread (Table~\ref{tab:exp3-ablation}). The two largest gains belong to a Grashof indicator and a transmission-angle summary, both describing motion feasibility, which the solve graph leaves to simulation.

\subsection{Relabelling joints leaves the solve graph identical and the baseline unrecognizable}
\label{sec:exp-invariance}

Permuting joint indices leaves the machine unchanged, so features that change under a permutation depend on the corpus's numbering convention. We generate 300 mechanisms at $N=4,\ldots,12$ with the compositional sampler, decompose each once, then re-decompose it under twenty random index permutations, carrying the motor pair and coupler joint through each.

\paragraph{Our representation is invariant under every relabelling and the baseline under almost none.} Across all 6,000 trials the fingerprint, the full Assur feature vector, and its structural sub-block match their pre-permutation values exactly at every size, and every relabelled mechanism still decomposed. The adjacency-plus-position baseline is unchanged in $0.68\%$ of the same trials, scoring $0.050$ at $N=4$ and exactly zero at every $N \geq 6$ (Table~\ref{tab:exp4-invariance}). That $N=4$ residue is a small-graph artifact, since on four joints a random permutation can still be the identity or an automorphism.

Invariance is earned by sorting dyads into a canonical order before any feature is written, since the solve-order search would otherwise break ties on the lowest joint index, and 242 of the 300 mechanisms contain at least one such tie. No joint index enters the group-order counts or the condensation-DAG hash.

%% file: sections/03_related_work.tex
\section{Related Work}
\label{sec:related}

\paragraph{Mechanism corpora and retrieval.}
LINKS delivered a large joint-level corpus of planar revolute mechanisms with simulated trajectories~\citep{HeyraniNobari2022Links}, and the Purwar Lab dataset adds four-, six-, and eight-bar families carrying higher-arity links and sliders~\citep{purwarlab2024dataset}. Surveys document the shift toward data-driven synthesis~\citep{PurwarGe2025Advancements, Han2025Review}, and a related line assembles linkages from available standard parts~\citep{escande2025relink}. LInK learns a joint representation over ten million mechanisms and warm-starts a numerical optimizer~\citep{HeyraniNobari2024LInK}, with contrastive extensions to mixed joint types~\citep{ParkJang2026Spring}. We build on those corpora, but our contribution is the abstraction above them. Learned embeddings abstract instances by similarity; our decomposition abstracts mechanisms by structure, reducing them to irreducible modules with compositional interfaces. This abstraction makes admissibility closed under composition, specifies the conditions for transfer (\S\ref{sec:theory}), makes acceptance decidable for designs absent from the corpora (\S\ref{sec:exp-validity}), and reduces the number of examples required at low budgets by half (\S\ref{sec:exp-regression}).

\paragraph{Keeping the discrete part valid.}
Published generators handle discrete validity by restricting topology to a library of types~\citep{Bolanos2026MechaFormer}, by appending Assur dyads under a grammar so that mobility holds inductively~\citep{Fogelson2023GCPHOLO, Lipson2008Evolutionary, Vermeer2018RL}, or by generating freely and filtering afterwards~\citep{Jadhav2026LinkD}. We apply the classical construction rule at a different point, since each of these systems constrains a graph it builds itself. The solve graph decomposes mechanisms that are given, which turns the rule into an acceptance test that every corpus record commits to, and \S\ref{sec:exp-validity} measures what construction alone delivers.

\paragraph{Decomposition and rigidity.}
The irreducible decomposition of pinned planar mechanisms and its relation to rigidity are established~\citep{ServatiusShaiWhiteley2010, ShaiSljokaWhiteley2013}, Laman's theorem characterizes generic planar bar-joint rigidity~\citep{Laman1970}, and the pebble game decides it exactly~\citep{JacobsHendrickson1997,10.1115/DETC2011-48340}. This literature asks structural questions of the modules it identifies, so what a learner gains from being handed them stays open, and it is what we measure. Our structural baseline in \S\ref{sec:exp-index} is Weisfeiler-Lehman hashing of the bare joint graph~\citep{Shervashidze2011WL}. Design systems that pair a language model with a simulator read a symbolic layer describing what the simulation already produced~\citep{Gandarela2026Symbolic, Gandarela2026Reflection, Gandarela2026RAPS}, while a decomposition is readable before any simulation runs.

\paragraph{Compositional structure and data efficiency.}
Exposing parts and their relations is a standard route to reuse and to learning from fewer examples, through relational inductive biases in graph networks~\citep{battaglia2018relational}, through modules recomposed across tasks~\citep{alet2018modular}, and, for molecules and programs, through decoders that assemble objects from individually valid pieces~\citep{Kusner2017GrammarVAE, Jin2018JTVAE}. In those settings the parts are designed, mined, or searched for, because the domain does not settle what they are. Planar mechanism theory settles it, so modules, interfaces, and solve order enter the input of a network whose architecture is held fixed, which separates what compositional structure contributes from what a relational architecture contributes. Read on the data axis, constraint-organized features hold a geometric-mean $1.8\times$ sample-efficiency advantage from 50 through 10,000 mechanisms, and the advantage closes by 40,000 (\S\ref{sec:exp-regression}), so the gain belongs to the regime where every label costs a simulation. Our contribution is this compositional representation at corpus scale, the interface conditions under which its modules support reuse (\S\ref{sec:theory}), and a measurement of what the compositionality is worth in training data.

%% file: sections/08_conclusion.tex
\section{Conclusion}
\label{sec:conclusion}

Solve graphs change what is explicit without changing the mechanism. They replace a flat view of joints with irreducible modules, typed interfaces, and a dataflow order for local constraint solves, which is what lets a design agent read the corpus as an action set whose moves carry checkable preconditions. Knowledge representation has long held that a formalism is judged by the commitments it makes explicit and by the inferences those commitments make cheap~\citep{BrachmanLevesque2004}, and for planar mechanisms the Assur decomposition supplies commitments that a stored adjacency matrix leaves implicit. Three consequences of that change are measured, with a fourth experiment controlling the third. In this context, composition stays inside a substantially more productive structural search space, with 10.02 million mechanisms admitting a more abstract structural representation, leading to a measurable low-data advantage over a flattened MLP. Structural validity and solve order are addressed before a simulator runs, while motion remains emergent.

%% file: sections/llm_disclosure.tex
\section{Generative AI Use Disclosure}
\label{sec:ai-disclosure}

In this work, we used generative AI tools to draft parts of the manuscript, to edit the text for readability, and to create and edit software code, and to assist with drafting informal proof sketches under manual supervision. The authors reviewed and revised all such material.

%% file: appendices/A_background.tex
\section{Background: Full Tutorial Treatment}
\label{app:background}

This appendix preserves the tutorial-style derivation of the kinematic and graph-theoretic background that motivates the Assur knowledge base. The main body (\S\ref{sec:background}) presents only the condensed version.

\subsection{Mechanisms as constraint systems}
\label{app:mechanisms}

A planar rigid body has three \emph{degrees of freedom} (DOF): two translational ($x$, $y$) and one rotational ($\theta$). When two bodies are connected by a \emph{revolute joint} (a pin), two relative DOF are removed (the joint fixes relative $x$ and $y$ but leaves $\theta$ free).

This maps cleanly onto a constraint-satisfaction problem. Given a mechanism with $L$ links (including the fixed \emph{ground} link) and $J$ revolute joints, the configuration space has dimension

\begin{equation}
  \boxed{\DOF = 3(L-1) - 2J}
  \label{eq:grubler}
\end{equation}

This is the planar Chebychev--Gr\"ubler count~\citep{Mueller2009GenericMobility}. Under the regularity and independence assumptions of the generic model, $\DOF=1$ means that one generalized coordinate (e.g.\ the input crank angle) locally drives all others. The count is structural: by itself it neither certifies that a particular geometry assembles nor that an assembled mechanism completes a full input cycle.

\begin{center}
\begin{tabular}{lcccc}
  \toprule
  Mechanism & $L$ & $J$ & $\DOF$ & Notes \\
  \midrule
  4-bar linkage         & 4 & 4 & 1 & standard crank--rocker \\
  Slider-crank          & 4 & 4 & 1 & piston--crank \\
  Pantograph (6-bar)    & 6 & 7 & 1 & copying mechanism \\
  Rigid triangle        & 3 & 3 & 0 & structure, not mechanism \\
  \bottomrule
\end{tabular}
\end{center}

\subsection{The 4-bar linkage}
\label{app:fourbar}

Figure~\ref{fig:fourbar} shows the canonical planar mechanism. Links are named
by their function: a \emph{crank} (driver) rotates continuously; the
\emph{coupler} transmits motion; the \emph{rocker} oscillates.

\begin{figure}[h]
\centering
\begin{tikzpicture}[scale=1.0]

  \coordinate (A) at (0,0);
  \coordinate (B) at (1.2,2.8);
  \coordinate (C) at (4.8,2.95);
  \coordinate (D) at (5.8,0);

  \draw[cGround, line width=1pt] (-0.4,0) -- (6.2,0);
  \foreach \x in {-0.2, 0.2, 0.6, 1.0, 1.6, 2.2, 2.8, 3.4, 4.0, 4.6,
                   5.0, 5.4, 5.8, 6.0}
    \draw[cGround, line width=0.6pt] (\x,0) -- (\x-0.3,-0.3);

  \begin{scope}[on background layer]
    \node[group box, draw=cDriver, fill=cBox, opacity=0.55,
          fit={($(A)+(-0.55,-0.55)$) ($(B)+(0.55,0.55)$)}]
          (driverbox) {};
  \end{scope}
  \node[above=2pt of driverbox, font=\small\bfseries, text=cDriver]
       {Driver (crank)};

  \begin{scope}[on background layer]
    \node[group box, draw=cDyad, fill=cBoxTeal, opacity=0.45,
          fit={($(B)+(-0.45,-0.5)$) ($(C)+(0.45,0.6)$) ($(D)+(0.55,0.55)$)}]
          (dyadbox) {};
  \end{scope}
  \node[above=2pt of dyadbox, font=\small\bfseries, text=cDyad, xshift=8pt]
       {Assur dyad (coupler + rocker)};

  \draw[driver link] (A) -- (B);
  \draw[dyad link]   (B) -- (C);
  \draw[dyad link]   (C) -- (D);

  \node[fixed joint] at (A) {};
  \node[joint, draw=cDriver] at (B) {};
  \node[joint, draw=cDyad]   at (C) {};
  \node[fixed joint] at (D) {};

  \node[node label, left=4pt]  at (A) {$A$};
  \node[node label, above left=2pt] at (B) {$B$};
  \node[node label, above right=2pt] at (C) {$C$};
  \node[node label, right=4pt] at (D) {$D$};

  \node[link label, text=cDriver, left=2pt]
    at ($(A)!0.5!(B)$) {crank};
  \node[link label, text=cDyad, above=3pt]
    at ($(B)!0.5!(C)$) {coupler};
  \node[link label, text=cDyad, right=2pt]
    at ($(C)!0.5!(D)$) {rocker};

  \draw[arr, cDriver] ($(A)+(0.65,0.1)$)
    arc[start angle=10, end angle=80, radius=0.65];

  \node[annotation, align=left, right=12pt] at (6.4, 2.0)
    {$L=4,\; J=4$\\[2pt] $\DOF = 3(3)-2(4) = 1$};

\end{tikzpicture}
\caption{A planar 4-bar linkage decomposed into its driver link and an Assur dyad. Fixed pivots $A$ and $D$ are grounded. The dyad's two \emph{external} joints are $B$ (position known from the driver) and $D$ (fixed); its one \emph{internal} joint is $C$.}
\label{fig:fourbar}
\end{figure}
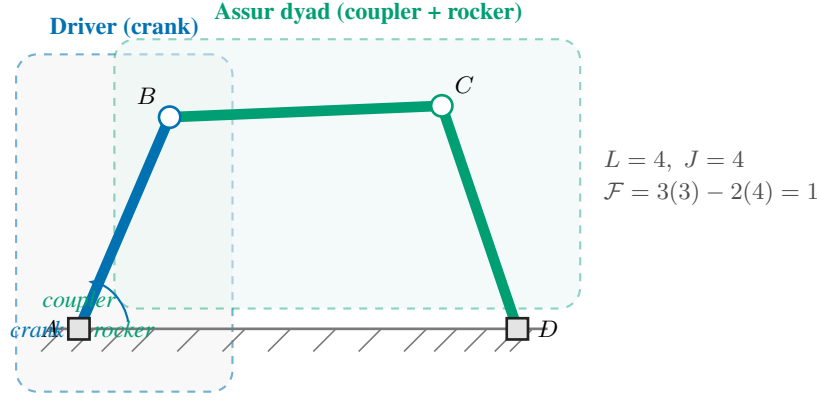

\subsection{Generic mobility and minimal pinned components}
\label{app:laman}

For a pinned subassembly of $n$ moving links and $j$ lower-pair constraints, the equality

\begin{equation}
  3n = 2j.
  \label{eq:laman}
\end{equation}

is the zero-mobility \emph{count}. It is necessary for generic isostaticity, but equality alone is not a rigidity certificate: the corresponding sparsity/minimality conditions must also hold~\citep{GraverServatiusServatius1993}. We therefore keep the generic count, the decomposition result, assembly at the stored geometry, and full-cycle motion as four distinct checks.

For the 4-bar as a whole: $3 \times 3 = 9 \neq 8 = 2 \times 4$, so the full mechanism is \emph{not} isostatic: it has one DOF, as required. But consider only the dyad sub-chain (coupler + rocker):

\[
  n = 2 \text{ links},\quad j = 3 \text{ joints (one internal + two external)}
  \implies 3(2) = 6 = 2(3). \checkmark
\]

\subsection{The pebble game}
\label{app:pebble-game}

Equation~\ref{eq:laman} totals freedoms and constraints over a whole
subassembly, so it remains satisfied by a subassembly that spends surplus
constraints in one region while leaving another region underconstrained.
Deciding generic rigidity therefore means verifying a count on every subset of
joints, and the \emph{pebble game} of \citet{JacobsHendrickson1997} performs all
of those checks in one incremental pass. We describe it from scratch here, since
the acceptance test of \S\ref{sec:acceptance} is stated in its terms.

Each joint of a planar bar--joint framework receives two pebbles, one for each
coordinate it could move along. A pebble resting on a joint is a freedom still
available there, and a pebble moved onto a bar is a freedom that bar has
consumed. Bars are inserted one at a time. A bar between joints $u$ and $v$ may
enter only when four pebbles can be gathered on $u$ and $v$, two at each
endpoint, because a bar in the plane removes two freedoms and the game requires
the two remaining ones to be visible. When an endpoint is short, the game
follows the bars already inserted outward in search of a free pebble at some
other joint and slides it back along the path on which it was found. This step,
path reversal, re-charges which bar is paid for by which freedom and alters no
total. If four pebbles are gathered, the new bar claims one of them and counts
as independent. If the search exhausts everything reachable from $u$ and $v$,
every freedom in that region is already spent, so the new bar is implied by the
bars already present and counts as redundant.

(In essence, the game is degree-of-freedom bookkeeping on the graph. Two tokens
per joint, one token spent by each independent bar, and a failed search for a
token proves that the region reached is already rigid.)

The pebbles still free when insertion finishes are the freedoms no bar removed,
which is the framework's generic degree-of-freedom count. With every joint free
to move, three pebbles remain exactly when the only motions left are the three
rigid motions of the plane, and Laman's theorem identifies that state with the
counting condition holding on every subset of
joints~\citep{Laman1970}. The name $(2,3)$ records those two numbers, two
pebbles per joint and the three a rigid framework keeps. In the pinned form the
acceptance test uses, the boundary joints enter first as a complete graph, so
the plane's three freedoms stay visible in the same tally, a module counts as
pinned-rigid when exactly three pebbles are free, and any surplus over three is
its residual freedoms. On the dyad of Figure~\ref{fig:dyad} the three joints
supply six pebbles, the boundary edge $p_1p_2$ and the two links $p_1q$ and
$qp_2$ each enter as independent, and three pebbles are left over, so the dyad
is pinned-rigid. The cost stays close to linear in the number of bars at these
sizes, which is what makes the test affordable on every module of the build, and
it is the algorithm the rigidity-theoretic treatment of Assur decomposition runs
as well~\citep{10.1115/DETC2011-48340}.

\subsection{Assur groups: definition}
\label{app:assur-groups}

\begin{quote}
  An \textbf{Assur group} is a minimal pinned isostatic component: it is generically rigid once its interface joints are pinned, and no proper admissible subcomponent has that property~\citep{ServatiusShaiWhiteley2010}.
\end{quote}

The simplest non-trivial Assur group is the \textbf{dyad}: two links sharing one internal revolute joint, with two external joints that attach to the rest of the mechanism (Figure~\ref{fig:dyad}).

\begin{figure}[h]
\centering
\begin{tikzpicture}[scale=1.1]

  \coordinate (P1) at (0, 0);
  \coordinate (Q)  at (1.8, 1.5);
  \coordinate (P2) at (4.0, 0);

  \begin{scope}[on background layer]
    \node[group box, draw=cDyad, fill=cBoxTeal, opacity=0.5,
          fit={($(P1)+(-0.5,-0.5)$) ($(Q)+(0,0.7)$) ($(P2)+(0.5,-0.5)$)}]
          {};
  \end{scope}

  \draw[cGround, line width=0.8pt] (-0.35,-0.25) -- (0.35,-0.25);
  \foreach \x in {-0.25,-0.05,0.15}
    \draw[cGround, line width=0.5pt] (\x,-0.25) -- (\x-0.2,-0.5);
  \draw[cGround, line width=0.8pt] (3.65,-0.25) -- (4.35,-0.25);
  \foreach \x in {3.75,3.95,4.15}
    \draw[cGround, line width=0.5pt] (\x,-0.25) -- (\x-0.2,-0.5);

  \draw[dyad link] (P1) -- (Q);
  \draw[dyad link] (Q) -- (P2);

  \node[fixed joint] at (P1) {};
  \node[joint, draw=cDyad, fill=white, line width=1.5pt] at (Q) {};
  \node[fixed joint] at (P2) {};

  \node[node label, below=10pt] at (P1) {ext.\ joint $p_1$};
  \node[node label, above=5pt] at (Q)  {int.\ joint $q$};
  \node[node label, below=10pt] at (P2) {ext.\ joint $p_2$};

  \node[link label, text=cDyad, above left=5pt]
    at ($(P1)!0.5!(Q)$) {link $\ell_1$};
  \node[link label, text=cDyad, above right=5pt]
    at ($(Q)!0.5!(P2)$) {link $\ell_2$};

  \node[annotation, align=center, below=30pt] at ($(P1)!0.5!(P2)$)
    {$n=2,\; j=3 \;\Longrightarrow\; 3n = 6 = 2j$ \quad (isostatic)};

\end{tikzpicture}
\caption{A dyad with interface arity two (and classical order two under the one-pair-per-boundary convention). Two links, three joints (one internal, two external). Given $p_1$ and $p_2$, the position of $q$ is found by intersecting two circles, typically yielding two assembly modes.}
\label{fig:dyad}
\end{figure}
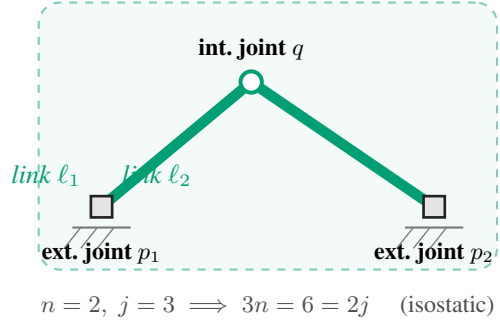

\begin{table}[h]
\centering
\begin{tabular}{ccccc}
  \toprule
  Stored interface arity & Links $n$ & Joints $j$ & Int.\ joints & Boundary IDs \\
  \midrule
  2 & 2 & 3 & 1 & 2 \\
  3 & 4 & 6 & 3 & 3 \\
  4 & 6 & 9 & 5 & 4 \\
  \bottomrule
\end{tabular}
\caption{Illustrative interface arities stored by this work. The value counts
  distinct boundary-joint identifiers and agrees with classical Assur order
  only under a one-to-one boundary-pair convention. It is not classical
  Assur class, which is recorded only with a named convention and independent
  validation. The equality $3n=2j$ is a generic count, not an independence,
  assembly, or full-cycle certificate.}
\label{tab:orders}
\end{table}

Once a mechanism has been peeled into groups
(\S\ref{sec:extraction-method}), contracting each group to a node and drawing
an edge wherever one group supplies a boundary joint to another gives the solve
graph of Figure~\ref{fig:technical-solve-graph}, whose topological traversals
are the dependency-respecting solve orders.

\begin{figure}[h]
 \centering
 \includegraphics[width=\linewidth]{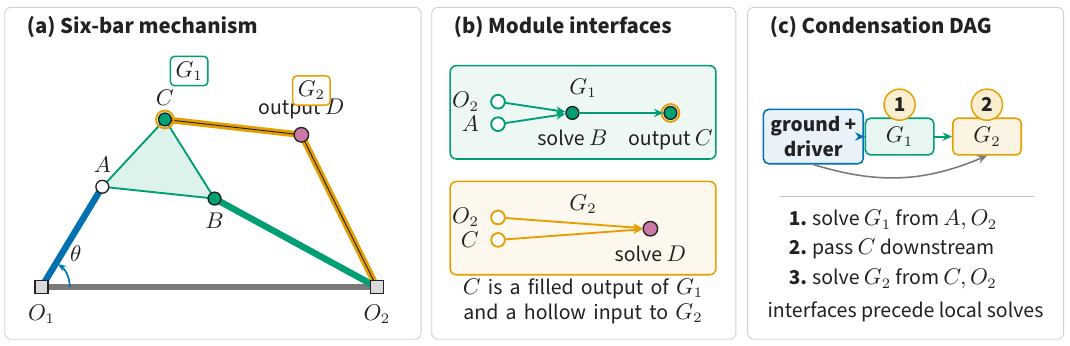}
 \caption{\textbf{A solve graph localizes the coupled computation in a
 multi-group mechanism.} Hollow joints are inputs and filled joints are
 outputs, and ground meets two moving links at the compound pin $O_2$. Numbers
 give one valid topological solve order. This is the six-bar of
 Figure~\ref{fig:dyad-peeling}b carried through to its condensation DAG.}
 \label{fig:technical-solve-graph}
\end{figure}

The mobility equation above uses the body--joint view (rigid links as bodies,
revolute joints as lower pairs). The binary-link LINKS source admits an
equivalent pinned bar--joint view (joint coordinates as vertices and binary
links as distance constraints); there the necessary component count is
$|E|=2|I|$ for internal vertices $I$. We do not substitute that bar--joint
count for the link-incidence model used by higher-arity or prismatic sources.

\subsection{A higher-class example}
\label{app:higher-class}

Figure~\ref{fig:condensation-dag} repeats the change of representation of Figure~\ref{fig:dyad-peeling} outside the dyadic regime, on a Stephenson-style mechanism whose second group has three mutually coupled internal joints. What the group makes visible that adjacency hides is the \emph{location} of the coupling. The rest of the mechanism factors into one small local solve per group, the driver and $G_1$ each fixing a single joint, and only the three internals of $G_2$ are coupled to each other. Adjacency alone gives a flat set of joints and bars: it neither names the dyad nor isolates the three joints whose coordinates must be solved together.

\begin{figure}[h]
\centering
\begin{minipage}[t]{0.42\textwidth}
  \centering
  \begin{tikzpicture}[scale=0.62, every node/.style={transform shape}]
    \coordinate (j0) at (0.0, 0.0);
    \coordinate (j7) at (5.0, 0.0);
    \coordinate (j1) at (1.5, 1.1);
    \coordinate (j2) at (3.4, 1.1);
    \coordinate (j3) at (1.7, 2.6);
    \coordinate (j4) at (3.3, 3.0);
    \coordinate (j5) at (4.4, 2.0);
    \draw[link, color=cGround!75, line width=1.3pt] (j0) -- (j1);
    \draw[link, color=cGround!75, line width=1.3pt] (j1) -- (j2);
    \draw[link, color=cGround!75, line width=1.3pt] (j7) -- (j2);
    \draw[link, color=cGround!75, line width=1.3pt] (j1) -- (j3);
    \draw[link, color=cGround!75, line width=1.3pt] (j2) -- (j4);
    \draw[link, color=cGround!75, line width=1.3pt] (j7) -- (j5);
    \draw[link, color=cGround!75, line width=1.3pt] (j3) -- (j4);
    \draw[link, color=cGround!75, line width=1.3pt] (j4) -- (j5);
    \draw[link, color=cGround!75, line width=1.3pt] (j5) -- (j3);
    \node[fixed joint] at (j0) {};
    \node[fixed joint] at (j7) {};
    \node[joint] at (j1) {};
    \node[joint] at (j2) {};
    \node[joint] at (j3) {};
    \node[joint] at (j4) {};
    \node[joint] at (j5) {};
  \end{tikzpicture}\\[4pt]
  {\footnotesize\textbf{(a) Adjacency view:} a flat set of joints and
  bars; the dyad and the class-III group are not visible, and the
  upper triangle is indistinguishable from three braced bars or one
  ternary rigid body.}
\end{minipage}\hfill
\begin{minipage}[t]{0.55\textwidth}
  \centering
  \begin{tikzpicture}[scale=0.62, every node/.style={transform shape}]
    \coordinate (j0) at (0.0, 0.0);
    \coordinate (j7) at (5.0, 0.0);
    \coordinate (j1) at (1.5, 1.1);
    \coordinate (j2) at (3.4, 1.1);
    \coordinate (j3) at (1.7, 2.6);
    \coordinate (j4) at (3.3, 3.0);
    \coordinate (j5) at (4.4, 2.0);
    \begin{scope}[on background layer]
      \fill[rigid body=cClassIII] (j3) -- (j4) -- (j5) -- cycle;
    \end{scope}
    \draw[driver link] (j0) -- (j1);
    \draw[dyad link]   (j1) -- (j2);
    \draw[dyad link]   (j7) -- (j2);
    \draw[class3 link] (j1) -- (j3);
    \draw[class3 link] (j2) -- (j4);
    \draw[class3 link] (j7) -- (j5);
    \node[fixed joint] at (j0) {};
    \node[fixed joint] at (j7) {};
    \node[joint, draw=cDriver] at (j1) {};
    \node[joint, draw=cDyad]   at (j2) {};
    \node[joint, draw=cClassIII] at (j3) {};
    \node[joint, draw=cClassIII] at (j4) {};
    \node[joint, draw=cClassIII] at (j5) {};
  \end{tikzpicture}\\[2pt]
  \begin{tikzpicture}[scale=0.7, every node/.style={transform shape},
                      >={Stealth[length=4pt]}]
    \node[dag node=cGround]   (g)  at (0,0)    {ground};
    \node[dag node=cDriver]   (d)  at (1.7,0)  {driver};
    \node[dag node=cClassII]  (g1) at (3.5,0)  {$G_1$ (II)};
    \node[dag node=cClassIII] (g2) at (5.5,0)  {$G_2$ (III)};
    \draw[arr] (g)  -- (d);
    \draw[arr] (g)  to[bend right=22] (g1.south west);
    \draw[arr] (d)  -- (g1);
    \draw[arr] (g)  to[bend right=30] (g2.south west);
    \draw[arr] (d)  to[bend left=14] (g2);
    \draw[arr] (g1) -- (g2);
  \end{tikzpicture}\\[2pt]
  {\footnotesize\textbf{(b) Assur view:} driver $\to$ dyad $G_1$ (class
  II) $\to$ class-III group $G_2$. Solving $G_1$ writes one joint
  (circle intersection); solving $G_2$ writes three joints
  \emph{coupled to each other}, surfacing exactly where the local
  solve stops factoring.}
\end{minipage}
\caption{The change of representation in the higher-class regime, on a
  Stephenson-style mechanism. \textbf{(a)} The adjacency view: seven
  joints and a set of bars; no group structure, and the upper triangle
  cannot be told apart from one ternary rigid body. \textbf{(b)} The
  Assur view: the driver determines the joint above $j_0$, the
  class-II dyad $G_1$ writes the joint between the two ground pivots,
  and the class-III group $G_2$ (amber rigid body, three coupled
  internal joints) writes the upper three joints at once. The
  condensation DAG (below right) makes the cascade and the location of
  the coupling explicit: the rest of the mechanism factors into one
  small local solve per group, and only the three internals of $G_2$
  are coupled to each other.}
\label{fig:condensation-dag}
\end{figure}
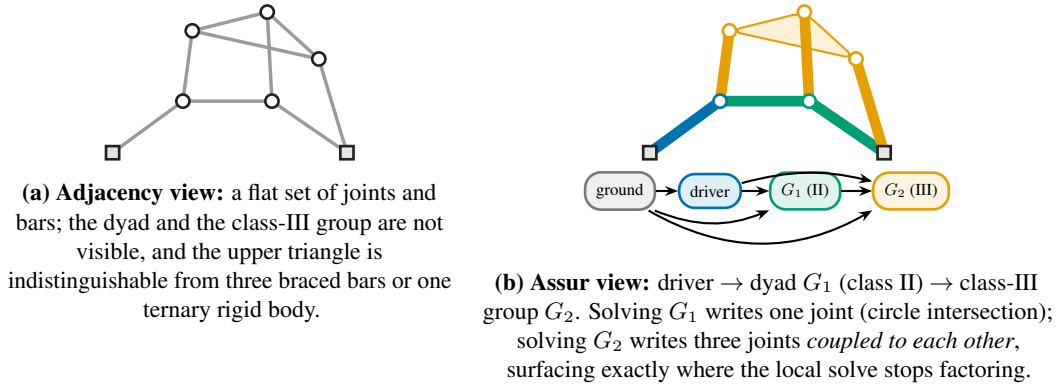





\subsection{The LINKS dataset in detail}
\label{app:links-dataset}

The published LINKS dataset reports one hundred million planar linkage mechanisms, with a reduced release of ten million mechanisms~\citep{HeyraniNobari2022Links}. These are upstream dataset sizes, not a claim that this work transformed every record. Each mechanism is a triple $(A, x_0, n)$:
\begin{itemize}[nosep]
  \item $A \in \{0,1\}^{N \times N}$ is the symmetric adjacency matrix
        of the joint graph.
  \item $x_0 \in \mathbb{R}^{N \times 2}$ holds the initial planar positions
        of all $N$ joints.
  \item $n \in \{0,1\}^{N}$ flags each joint as moving ($0$) or grounded
        ($1$).
\end{itemize}
The motor is fixed by convention: joints $0$ and $1$ are connected, joint $0$ is always grounded, and the link $0{-}1$ acts as the driver crank.


\subsection{The Purwar Lab dataset in detail}
\label{app:purwarlab-dataset}

The published Purwar Lab dataset reports approximately three million single-DOF planar four-, six-, and eight-bar linkage mechanisms with open and closed coupler curves~\citep{purwarlab2024dataset}. The build streams $2{,}966{,}754$ of them, divided into $602{,}079$ four-bar, $397{,}525$ six-bar, and $1{,}967{,}150$ eight-bar mechanisms (Appendix~\ref{app:kb-census}, which also gives the release link). Each mechanism is given as:
\begin{itemize}[nosep]
  \item $B \in \{0,1\}^{L \times N}$ is the link-incidence matrix, one row per
        rigid link over $N$ joints, with row $0$ the ground link, so a ternary
        or quaternary body is a single row and never a ring of bars.
  \item $\tau : J \to \{\mathrm{R}, \mathrm{P}\}$ marks each joint revolute or
        prismatic, read off the named topology (\texttt{RRRR}, \texttt{RRRP},
        \texttt{RRPR}, \texttt{PRPR}), which is what makes the slider families
        reachable at all.
  \item one coupler curve per record, travelling inline with the record and
        traced by a point carried on a link.
\end{itemize}
The motor is fixed by the source topology table, which supplies the grounded joints on incidence row $0$ together with the driver pair. No joint-coordinate array is released, so the pose $x_0$ is decoded from the source file name and the resulting design frame is accepted only once the traced point maps through it onto the stored curve (\S\ref{sec:reconstructability}).

%% file: appendices/H_benchmark_tables.tex
\section{Full Benchmark Tables}
\label{app:benchmark-tables}

These tables reproduce the completed experiment summaries used in \S\ref{sec:benchmarks}.

\subsection{Compositional validity}

Section~\ref{sec:exp-validity} reports the decline in the empirical branch and leaves its mechanism here. Extra stationary pivots can make the distance between two boundary joints vary, over a revolution, beyond what a randomly sized dyad can close. The attachment stays structurally admissible throughout, because admissibility reads only the groups and their dependencies, and it is the sampled lengths that fail. This is why the decline is a property of how geometry is drawn, which leaves the construction rule intact, and why the oracle records its geometric condition separately from its three structural ones.

\input{tables/experiments/exp1_validity_rates}

\subsection{Structural coarsening}

\input{tables/experiments/exp2_fingerprint_summary}

\subsection{Curve-regression learning curve}

The $p$ column tests ``Assur is better,'' so it reads near $1$ on the last two rows and separates them only under the reverse test. At 40,000 examples the $-0.0004$ difference sits inside fold noise, with a corrected paired interval of $[-0.0009, +0.0001]$ and a one-sided probability of $0.068$ for adjacency-better. At the full pool the $-0.0014$ difference is resolved, at $6.6\times10^{-6}$ with all ten folds favoring adjacency.

\input{tables/experiments/exp3_learning_curve}

\subsection{The default feature vector}
\label{app:assur-default-vector}

Section~\ref{sec:exp-regression} and the ablation below both speak of a
\emph{default} vector, which is the eight blocks named in
Table~\ref{tab:exp3-ablation}, $86$ numbers in all, each of them independent of pose, of absolute scale, and of the joint numbering. Group-order counts (\texttt{order\_count}, $3$) record how many groups of each order the decomposition returned. Condensation-DAG shape (\texttt{dag\_shape}, $6$) records its depth, its group and fixed-joint counts, and the three leading singular values of the group-by-group adjacency matrix. Link lengths (\texttt{lengths}, $37$) are the complete set in solve order, divided by the crank length, and assembly signs (\texttt{branches}, $20$) carry one branch per dyad in that same order, which together pin the configuration that lengths alone leave ambiguous. The tracer dyad contributes its own triangle (\texttt{coupler\_tri}, $5$) and the tracer's angle, area and perpendicular offset on the coupler body (\texttt{coupler\_shape}, $3$). Path geometry (\texttt{path\_geom}, $9$) holds the per-stage lengths along the chain from driver to coupler with that chain's total reach, and shape ratios (\texttt{shape\_ratios}, $3$) hold dimensionless link proportions. The structural sub-block of \S\ref{sec:exp-invariance} is the part that reads off the decomposition with no geometry at all, the order counts and the DAG-shape block together with a gated DAG-branching block, $13$ numbers.

\subsection{Assur-feature ablation}

Section~\ref{sec:exp-regression} states the outcome and leaves the magnitudes here. The condensation-DAG block is the only clear contributor at $+0.00219 \pm 0.00057$ Chamfer, and the branch block is the only one whose removal helps, lowering error by $0.00209 \pm 0.00061$. The remaining six default blocks sit inside fold noise, and the structural sub-block on its own scores $0.10375$ against $0.10576$ for the full default vector, so the vector as shipped is carried by its condensation-DAG component. Among blocks outside the default, a Grashof indicator gains $0.00397$ and a transmission-angle summary $0.00295$, both larger than anything currently included. Each describes motion feasibility, which the solve graph leaves for simulation to settle, so neither enters the default vector on this evidence alone.

\input{tables/experiments/exp3_ablation_summary}

\subsection{Relabelling invariance}
Mechanisms are generated by the Experiment~1 compositional sampler, decomposed once, then re-decomposed under a random permutation of all joint indices with the motor pair and coupler joint remapped through the same permutation. Rates are the fraction of trials whose descriptor is bitwise unchanged.

\input{tables/experiments/exp4_invariance}

%% file: tables/experiments/exp1_validity_rates.tex

\begin{table}[t]
\caption{Experiment~1 validity rates from 50,000 samples per procedure
per size. Procedure~A is the compositional dyad sampler; procedure~B is
the matched adjacency sampler.}
\label{tab:exp1-results}
\centering
\resizebox{\linewidth}{!}{
\begin{tabular}{llrrrrrrrrrr}
\toprule
Branch & $N$ & $A$ overall & $B$ overall & $A$ geom. & $B$ geom. & $A$ conn & $B$ conn & $A$ path & $B$ path & $A$ dof & $B$ dof \\
\midrule
theoretical & 4 & 1.000 & 0.405 & 1.000 & 0.405 & 1.000 & 0.990 & 1.000 & 0.405 & 1.000 & 0.411 \\
theoretical & 5 & 1.000 & 0.204 & 1.000 & 0.204 & 1.000 & 0.974 & 1.000 & 0.204 & 1.000 & 0.277 \\
theoretical & 6 & 1.000 & 0.108 & 1.000 & 0.108 & 1.000 & 0.958 & 1.000 & 0.108 & 1.000 & 0.215 \\
theoretical & 7 & 1.000 & 0.061 & 1.000 & 0.061 & 1.000 & 0.940 & 1.000 & 0.061 & 1.000 & 0.178 \\
theoretical & 8 & 1.000 & 0.035 & 1.000 & 0.035 & 1.000 & 0.924 & 1.000 & 0.035 & 1.000 & 0.155 \\
empirical & 8 & 0.284 & 0.001 & 0.284 & 0.017 & 1.000 & 0.782 & 1.000 & 0.050 & 1.000 & 0.124 \\
empirical & 9 & 0.223 & 0.001 & 0.223 & 0.008 & 1.000 & 0.758 & 1.000 & 0.029 & 1.000 & 0.123 \\
empirical & 10 & 0.180 & 0.001 & 0.180 & 0.004 & 1.000 & 0.745 & 1.000 & 0.019 & 1.000 & 0.117 \\
empirical & 11 & 0.146 & 0.000 & 0.146 & 0.002 & 1.000 & 0.729 & 1.000 & 0.012 & 1.000 & 0.110 \\
empirical & 12 & 0.121 & 0.000 & 0.121 & 0.001 & 1.000 & 0.711 & 1.000 & 0.007 & 1.000 & 0.107 \\
\bottomrule
\end{tabular}
}
\end{table}

%% file: tables/experiments/exp2_fingerprint_summary.tex

\begin{table}[t]
\caption{Experiment~2 hash and fingerprint summary. Adjacency hashes
use the bare joint graph; solve-graph fingerprints use group-order
counts and the order-labelled condensation DAG.}
\label{tab:exp2-fingerprint}
\centering
\begin{tabular}{lr}
\toprule
Quantity & Value \\
\midrule
Mechanisms hashed & 10,020,000 \\
Distinct adjacency hashes $D_A$ & 1,173,078 \\
Distinct Assur fingerprints $D_F$ & 908,186 \\
Coarsening factor $D_A/D_F$ & 1.29 \\
Mean mechanisms per Assur fingerprint & 11.03 \\
Largest Assur-fingerprint bucket & 330,809 \\
Max adjacency hashes in one Assur fingerprint & 300 \\
\bottomrule
\end{tabular}
\end{table}

%% file: tables/experiments/exp3_learning_curve.tex

\begin{table}[t]
\caption{Experiment~3 learning curve for the 128-point output-joint
coupler curve. Scores are mean symmetric Chamfer distances
$\pm$ standard deviations over 10 stratified fold units. The
$p$ column is the one-sided Nadeau--Bengio-corrected paired-test
probability for ``Assur is better.''}
\label{tab:exp3-results}
\centering
\small
\begin{tabular}{rccrr}
\toprule
Train $n$ & Adj. Chamfer & Assur Chamfer & Adj. $-$ Assur & NB $p$ \\
\midrule
50 & 0.1822 $\pm$ 0.0290 & 0.1632 $\pm$ 0.0046 & 0.0190 & 0.4812 \\
100 & 0.1602 $\pm$ 0.0038 & 0.1526 $\pm$ 0.0029 & 0.0076 & 0.4122 \\
200 & 0.1473 $\pm$ 0.0022 & 0.1409 $\pm$ 0.0024 & 0.0064 & 0.3609 \\
500 & 0.1359 $\pm$ 0.0022 & 0.1294 $\pm$ 0.0018 & 0.0065 & 0.3310 \\
1,000 & 0.1274 $\pm$ 0.0014 & 0.1221 $\pm$ 0.0015 & 0.0052 & 0.2423 \\
2,000 & 0.1196 $\pm$ 0.0008 & 0.1158 $\pm$ 0.0016 & 0.0038 & 0.0866 \\
4,000 & 0.1131 $\pm$ 0.0008 & 0.1101 $\pm$ 0.0007 & 0.0030 & 0.0095 \\
8,000 & 0.1069 $\pm$ 0.0008 & 0.1052 $\pm$ 0.0007 & 0.0017 & 0.0060 \\
10,000 & 0.1051 $\pm$ 0.0008 & 0.1038 $\pm$ 0.0007 & 0.0013 & 0.0240 \\
20,000 & 0.1002 $\pm$ 0.0008 & 0.1004 $\pm$ 0.0006 & -0.0002 & 0.7193 \\
40,000 & 0.0964 $\pm$ 0.0005 & 0.0980 $\pm$ 0.0006 & -0.0016 & 1.0000 \\
72,013 & 0.0942 $\pm$ 0.0005 & 0.0965 $\pm$ 0.0007 & -0.0023 & 1.0000 \\
\bottomrule
\end{tabular}
\end{table}

%% file: tables/experiments/exp3_ablation_summary.tex
\begin{table}[t]
\centering
\begin{tabular}{llrr}
\toprule
Comparison & Block & Gain & Std. \\
\midrule
Default leave-one-out & \texttt{order\_count} & 0.00032 & 0.00074 \\
Default leave-one-out & \texttt{dag\_shape} & 0.00219 & 0.00057 \\
Default leave-one-out & \texttt{lengths} & 0.00022 & 0.00049 \\
Default leave-one-out & \texttt{branches} & -0.00209 & 0.00061 \\
Default leave-one-out & \texttt{coupler\_tri} & -0.00016 & 0.00074 \\
Default leave-one-out & \texttt{coupler\_shape} & 0.00017 & 0.00063 \\
Default leave-one-out & \texttt{path\_geom} & -0.00039 & 0.00048 \\
Default leave-one-out & \texttt{shape\_ratios} & -0.00048 & 0.00067 \\
\midrule
Candidate add & \texttt{grashof} & 0.00397 & 0.00065 \\
Candidate add & \texttt{coupler\_param} & -0.00036 & 0.00049 \\
Candidate add & \texttt{transmission} & 0.00295 & 0.00089 \\
Candidate add & \texttt{reach\_extent} & -0.00006 & 0.00056 \\
Candidate add & \texttt{loop\_spectral} & 0.00023 & 0.00069 \\
Candidate add & \texttt{interactions} & 0.00106 & 0.00066 \\
\bottomrule
\end{tabular}
\caption{Experiment~3 feature-ablation summary at train $n=5{,}000$ over 10 fold units. Default rows report leave-one-out importance; candidate rows report add-one-block gain. Positive values mean lower Chamfer relative to the corresponding baseline.}
\label{tab:exp3-ablation}
\end{table}

%% file: tables/experiments/exp4_invariance.tex
\begin{table}[t]
\centering
\begin{tabular}{rrrrrr}
\toprule
$N$ & Trials & Fingerprint & Assur & Assur struct. & Adjacency \\
\midrule
4 & 760 & 1.000 & 1.000 & 1.000 & 0.050 \\
5 & 580 & 1.000 & 1.000 & 1.000 & 0.005 \\
6 & 620 & 1.000 & 1.000 & 1.000 & 0.000 \\
7 & 620 & 1.000 & 1.000 & 1.000 & 0.000 \\
8 & 640 & 1.000 & 1.000 & 1.000 & 0.000 \\
9 & 620 & 1.000 & 1.000 & 1.000 & 0.000 \\
10 & 680 & 1.000 & 1.000 & 1.000 & 0.000 \\
11 & 800 & 1.000 & 1.000 & 1.000 & 0.000 \\
12 & 680 & 1.000 & 1.000 & 1.000 & 0.000 \\
\midrule
All & 6,000 & 1.000 & 1.000 & 1.000 & 0.007 \\
\bottomrule
\end{tabular}
\caption{Experiment~4 relabelling invariance: the fraction of trials left unchanged by a random joint permutation. 300 generated mechanisms, 20 permutations each. The adjacency column is the Experiment~3 flattened adjacency-plus-position baseline, which is keyed by joint index and is expected to change.}
\label{tab:exp4-invariance}
\end{table}

%% file: appendices/J_theory.tex
\section{Conditions and Boundaries for Predictive Composition}
\label{app:theory}

The central question of Section~\ref{sec:theory} is what makes a familiar
constraint module reusable in a new mechanism. A dyad illustrates the idea:
its internal joint is found by intersecting two circles. The same conditional
rule applies in many assemblies, while the boundary positions, bar lengths,
and assembly branch specify each instance. Composition reuses that rule and
passes its solution to later modules.

We first ask when accurate local predictions remain accurate after this
composition. The transfer bound separates familiarity of local inputs from
the amplification of local errors. Two physical constructions then explain
why amplification and loss of smoothness can occur even in chains of simple
dyads. Next, we ask what information must reach each module: a port-assignment
example and a dimension bound show how compression can merge different
behaviours. Finally, we distinguish the information retained by a lossless
decomposition from the computational structure it exposes, and show why
undirected graph width does not capture that structure.

The proofs use local implicit solves and the established Assur decomposition
\citep{ServatiusShaiWhiteley2010,ShaiSljokaWhiteley2013}; the transfer argument
builds on the broader study of compositional prediction.

\begin{figure}[htbp]
 \centering
 \includegraphics[width=\linewidth]{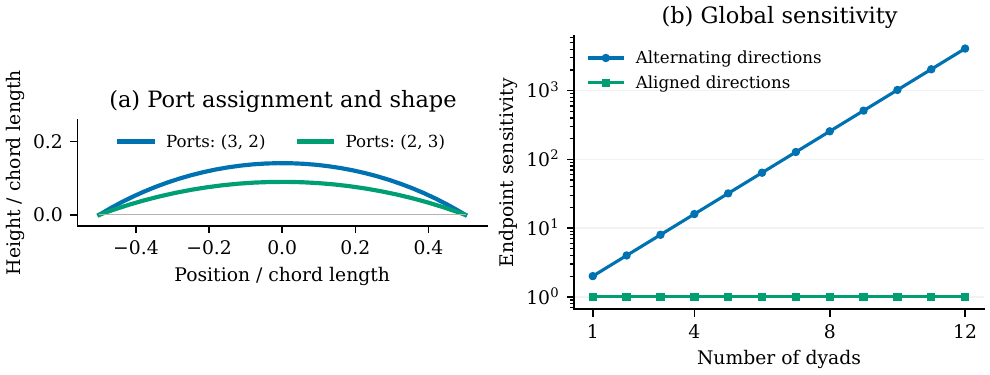}
 \caption{\textbf{What local structure alone leaves unresolved.}
 (a) Exchanging lengths between the boundary ports of a four-bar dyad
 changes its output arc, shown after normalizing both chords to unit length
 (Proposition~\ref{prop:port-information}).
 (b) Dyad chains with identical local Jacobian singular values have
 endpoint configuration sensitivities $1$ and $2^m$ at their reference
 poses (Theorem~\ref{thm:theory-dyad-sensitivity-separation}).
 Panel (a) shows why a module needs the correspondence between constraints
 and boundary ports; panel (b) shows why the directions in which modules
 transmit motion matter. Both are analytic constructions.
 The included script reproduces the figure and checks the constructions
 against circle-intersection solves.}
 \label{fig:theory-separations}
\end{figure}

\input{appendices/theory/transfer}
\input{appendices/theory/stability}
\input{appendices/theory/singular-depth}
\input{appendices/theory/information}
\input{appendices/theory/structure}

\paragraph{Consequences for representation and evaluation.}
The results give a connected account of predictive reuse. Interfaces must
retain distinctions visible in the target; training must cover the local
conditions encountered at those interfaces; and downstream geometry
determines how much local error matters. A solve decomposition makes these
quantities explicit, even when the undirected graph is large or has high
treewidth. They suggest concrete axes for future compositional evaluations:
vary interface coverage and propagation sensitivity separately, while
holding the module vocabulary fixed. The present common-MLP experiment
addresses the preceding representational question---whether exposing
constraint organization helps a fixed learner predict behaviour from
limited data.

%% file: appendices/theory/transfer.tex
\subsection{Generalization Through Module Interfaces}
\label{app:interface-transfer}

An assembly may be new even though each conditional problem it contains is
familiar. For example, a dyad predictor can be reused wherever its boundary
motion, geometry, and assembly branch resemble those seen during training.
The difficulty is that an upstream prediction also supplies the boundary
conditions for later modules. We therefore separate two sources of error:
the error a module makes with the correct inputs, and the additional error
caused by imperfect inputs arriving from its parents. The first is governed
by local coverage; the second by propagation through the solve graph.

\subsubsection{Spaces, domains, and error propagation}

The setup below formalizes this separation. The input and output spaces
allow a module to predict an entire path, so that the same argument applies
to mechanism motion as well as to a single configuration.

Fix a finite solve DAG with nodes $V=\{1,\ldots,m\}$ in topological order.
Write $\operatorname{pa}(v)$ for the parents of $v$, and $t(v)$ for its module
type.  For each type $t$, let $\mathsf X_t$ be a separable Banach space of
module outputs and $(\mathsf Z_t,d_t)$ a Polish space of module inputs.
For trajectory prediction one may take
\[
 \mathsf X_t=C([0,1];\mathbb R^{d_t^{\rm out}}),
 \qquad
 \|x\|_{\mathsf X_t}=\sup_{\theta\in[0,1]}\|x(\theta)\|_2.
\]
Here $\theta$ is a common driving parameter, not an independently resampled
arc-length coordinate.  The input space includes the full boundary paths,
geometric parameters, and branch information.  The input metric is fixed
before stating a transport bound; any rescaling of its coordinates changes
the associated Lipschitz constants.

For each type, a measurable exact conditional solver $S_t:D_t\to\mathsf X_t$ is
defined on an admissible Borel set $D_t\subseteq\mathsf Z_t$.
Its measurable learned counterpart $\widehat S_t:U_t\to\mathsf X_t$ is defined on a
Borel set $U_t\supseteq D_t$.  Branch-conditioned solver charts may be
included in the type index.  This avoids assuming a single-valued exact
solver across incompatible assembly branches.

Let $\omega$ collect the random mechanism geometry, source motion, and
other exogenous data, on a probability space $(\Omega,\mathcal F,\mathbb P)$.
The measurable routing map $Q_v$ supplies the inputs to occurrence $v$:
\begin{align}
 z_v&=Q_v\bigl(\eta_v,(x_u)_{u\in\operatorname{pa}(v)}\bigr),
 & x_v&=S_{t(v)}(z_v),\label{eq:transfer-true}\\
 \widehat z_v&=Q_v\bigl(\eta_v,(\widehat x_u)_{u\in\operatorname{pa}(v)}\bigr),
 & \widehat x_v&=\widehat S_{t(v)}(\widehat z_v).
 \label{eq:transfer-learned}
\end{align}
All quantities depend on $\omega$.  The same observed exogenous input
$\eta_v$ is used in the two recursions; uncertainty in that input is not
included in the bounds below.

We require $z_v\in D_{t(v)}$ almost surely.  To compare the two recursions,
also consider every hybrid input obtained by independently choosing either
$x_u$ or $\widehat x_u$ for each parent.  Every such input must lie in
$U_{t(v)}$.  The exact solver need not be defined at these hybrid inputs.
For two hybrid tuples differing only in parent $u$, assume
\begin{equation}
 \left\|
 \widehat S_{t(v)}\bigl(Q_v(\eta_v,\ldots,a_u,\ldots)\bigr)
 -\widehat S_{t(v)}\bigl(Q_v(\eta_v,\ldots,b_u,\ldots)\bigr)
 \right\|_{\mathsf X_{t(v)}}
 \le A_{vu}\|a_u-b_u\|_{\mathsf X_{t(u)}}.
 \label{eq:transfer-edge-gain}
\end{equation}
The constants are deterministic and uniform over the stated evaluation
population.  Set $A_{vu}=0$ when $u\notin\operatorname{pa}(v)$.
Thus $A$ is a nonnegative strictly lower-triangular matrix.  Bounds on
rectangular domains of admissible parent states imply
Eq.~\eqref{eq:transfer-edge-gain}, but are stronger than needed.

The hybrid inputs let us replace predicted parent outputs one at a time.
This attributes each incoming error to a specific dependency, with $A_{vu}$
bounding its effect. They are comparison points for the learned predictor;
requiring them to be physically feasible exact assemblies is unnecessary.

Write $x=(x_1,\ldots,x_m)$ and define $\widehat x$ analogously.
Let a measurable readout produce $y=R(\eta,x)$ and
$\widehat y=R(\eta,\widehat x)$ in a separable Banach space $\mathsf Y$.  Assume
\begin{equation}
 \|\widehat y-y\|_{\mathsf Y}
 \le\sum_{v=1}^m c_v\|\widehat x_v-x_v\|_{\mathsf X_{t(v)}},
 \qquad c_v\ge0.
 \label{eq:transfer-readout}
\end{equation}
Selecting a tracer coordinate is one such readout.  Define
\begin{equation}
 w^\top=c^\top(I-A)^{-1}
        =c^\top\sum_{j=0}^{m-1}A^j.
 \label{eq:transfer-path-weights}
\end{equation}
Each $w_v$ sums products of edge gains over directed paths from $v$ to the
readout, including the direct readout contribution $c_v$.

For intuition, an error passing through two successive dependencies with
gains $a$ and $b$ contributes at most $ab$ times its original magnitude.
When several paths reach the output, their contributions add in the bound.
The weight $w_v$ collects this downstream influence for occurrence $v$.

\begin{proposition}[Propagation from errors at true module inputs]
\label{prop:transfer-propagation}
Under the domain and gain assumptions above, let
\[
 \delta_t(z)=\|\widehat S_t(z)-S_t(z)\|_{\mathsf X_t},
 \qquad z\in D_t.
\]
Then, almost surely,
\begin{equation}
 \|\widehat y-y\|_{\mathsf Y}
 \le\sum_{v=1}^m w_v\delta_{t(v)}(z_v).
 \label{eq:transfer-pointwise}
\end{equation}
Consequently, if $\nu_v$ is the law of the true input $z_v$,
\begin{equation}
 \left(\mathbb E\|\widehat y-y\|_{\mathsf Y}^2\right)^{1/2}
 \le\sum_{v=1}^m w_v
       \|\delta_{t(v)}\|_{L^2(\nu_v)}.
 \label{eq:transfer-rms-propagation}
\end{equation}
\end{proposition}

\begin{proof}
Put $e_v=\|\widehat x_v-x_v\|_{\mathsf X_{t(v)}}$ and
$b_v=\delta_{t(v)}(z_v)$.  Insert
$\widehat S_{t(v)}(z_v)$ between the true and predicted outputs.
Changing one parent at a time from its true to its predicted output uses
only the stipulated hybrid inputs, so the triangle inequality and
Eq.~\eqref{eq:transfer-edge-gain} give
\[
 e_v\le b_v+\sum_{u\in\operatorname{pa}(v)}A_{vu}e_u.
\]
In vector form $e\le b+Ae$, componentwise.  Iteration terminates because
$A^m=0$, giving $e\le\sum_{j=0}^{m-1}A^jb$.
Multiplication by the nonnegative readout vector $c$ proves
Eq.~\eqref{eq:transfer-pointwise}.  Taking scalar $L^2(\mathbb P)$ norms
and using Minkowski's inequality proves
Eq.~\eqref{eq:transfer-rms-propagation}.  No independence between errors
at different occurrences is required.
\end{proof}

\subsubsection{Transport from training interfaces to new assemblies}

The propagation result assumes we know a module's error on the inputs of
the new assembly. Training instead characterizes its error under a reference
distribution. The next step connects the two: if nearby local inputs have
similar prediction errors, shifting the input distribution incurs a
controlled additional cost. The Wasserstein distance $W_2$ expresses the
smallest RMS displacement needed to match inputs from the two distributions.

Let $\mu_t$ be a reference training distribution of complete local inputs
for type $t$, supported on $D_t$.  The quantity
\begin{equation}
 \epsilon_t=\|\delta_t\|_{L^2(\mu_t)}
 \label{eq:transfer-local-risk}
\end{equation}
is a population error of the learned module, not its empirical training
error.  A separate statistical bound may control it from a finite sample.
The following statements can also be read conditionally on fitted module
predictors and any event on which their local error bounds hold.

\begin{theorem}[Generalization controlled by local interface transport]
\label{thm:interface-transport}
Assume the conditions of Proposition~\ref{prop:transfer-propagation}.
Suppose $\mu_t$ and $\nu_v$ have finite second moments in their respective
input metrics and, on a set supporting both measures, the scalar residual
$\delta_t$ is $K_t$-Lipschitz.  Then
\begin{equation}
 \left(\mathbb E\|\widehat y-y\|_{\mathsf Y}^2\right)^{1/2}
 \le
 \sum_{v=1}^m w_v
 \left[\epsilon_{t(v)}+
 K_{t(v)}W_2(\nu_v,\mu_{t(v)})\right],
 \label{eq:interface-transport-bound}
\end{equation}
where
\[
 W_2(\nu,\mu)
 =\inf_{\pi\in\Pi(\nu,\mu)}
   \left(\int d_t(z,z')^2\,\mathrm d\pi(z,z')\right)^{1/2}.
\]
It suffices for the exact and learned module maps to be Lipschitz on that
set with constants $L_t$ and $\widehat L_t$, respectively; one may then
take $K_t=L_t+\widehat L_t$.
\end{theorem}

\begin{proof}
Fix an occurrence $v$ and write $t=t(v)$.  For any coupling
$(Z,Z')\sim\pi\in\Pi(\nu_v,\mu_t)$, the Lipschitz condition gives
$\delta_t(Z)\le\delta_t(Z')+K_td_t(Z,Z')$ almost surely.
Minkowski's inequality therefore yields
\[
 \|\delta_t\|_{L^2(\nu_v)}
 \le \epsilon_t+
 K_t\left(\int d_t(z,z')^2\,\mathrm d\pi(z,z')\right)^{1/2}.
\]
Taking the infimum over couplings proves the local transport bound;
existence of an optimal coupling is not needed.  Substitute these bounds
in Eq.~\eqref{eq:transfer-rms-propagation}.
Finally, the reverse triangle inequality gives
\[
 |\delta_t(z)-\delta_t(z')|
 \le\|\widehat S_t(z)-\widehat S_t(z')\|_{\mathsf X_t}
    +\|S_t(z)-S_t(z')\|_{\mathsf X_t},
\]
which proves the stated sufficient condition for $K_t$.
\end{proof}

An alternative is useful when the new assembly encounters the same local
regimes with different frequencies. A bounded density ratio limits how much
more often any regime can occur, giving the following transfer bound.

\begin{corollary}[Density-ratio transfer]
\label{cor:interface-density}
Under Proposition~\ref{prop:transfer-propagation}, suppose instead that
$\nu_v\ll\mu_{t(v)}$ and
$\mathrm d\nu_v/\mathrm d\mu_{t(v)}\le C_v<\infty$ almost everywhere.
Then
\begin{equation}
 \left(\mathbb E\|\widehat y-y\|_{\mathsf Y}^2\right)^{1/2}
 \le\sum_{v=1}^m w_v\sqrt{C_v}\,\epsilon_{t(v)}.
 \label{eq:interface-density-bound}
\end{equation}
No residual Lipschitz assumption is required for this corollary.
\end{corollary}

\begin{proof}
For $t=t(v)$,
\[
 \int\delta_t^2\,\mathrm d\nu_v
 =\int\delta_t^2
       \frac{\mathrm d\nu_v}{\mathrm d\mu_t}\,\mathrm d\mu_t
 \le C_v\epsilon_t^2.
\]
Take square roots and apply
Eq.~\eqref{eq:transfer-rms-propagation}.
\end{proof}

These bounds permit new dependency graphs.  More precisely, if the same
module bank is used on a family of graphs and the right-hand side of
Eq.~\eqref{eq:interface-transport-bound} is at most $B$ for every graph,
the RMS bound $B$ holds uniformly on that family.  If a random graph has
conditional bound $B_G$, its unconditional RMS error is at most
$(\mathbb E_G B_G^2)^{1/2}$, by the law of total expectation.
The number of possible graphs does not appear separately; their effects
enter through local input laws and path weights.

\paragraph{Physical units and normalized module charts.}
Reuse can also exploit a shared coordinate system. Translating, rotating,
and scaling a dyad into a common frame lets one predictor handle many
placements. To propagate its error through the physical mechanism, we must
then account for the conversion back to world coordinates.
For a dyad, let $d(\theta)=\|p_2(\theta)-p_1(\theta)\|>0$, and let
$R(\theta)$ be the rotation whose first column is
$(p_2(\theta)-p_1(\theta))/d(\theta)$.  A normalized solver $s_t$ receives
the ratio paths $\rho_i=\ell_i/d$ and the branch.  The physical solver is
\[
 S_t(z)(\theta)=p_1(\theta)+d(\theta)R(\theta)
                      s_t(\chi(z))(\theta).
\]
Its learned counterpart uses the same reconstruction formula with
$\widehat s_t$.  At a true input, if
$\sup_\theta d(\theta)\le d_{\max}$, then
\begin{equation}
 \delta_t(z)\le d_{\max}
 \|\widehat s_t(\chi(z))-s_t(\chi(z))\|_\infty.
 \label{eq:transfer-chart}
\end{equation}
This follows pointwise from orthogonality of $R$ and then by taking a
supremum.  Thus the local term in
Eq.~\eqref{eq:transfer-rms-propagation} can alternatively be bounded by
\[
 d_{\max}\left[
 \overline\epsilon_t+
 \overline K_tW_2(\chi_\#\nu_v,\overline\mu_t)
 \right],
\]
where $\overline\mu_t$ is the training law of normalized inputs,
$\overline\epsilon_t$ is the normalized local RMS error, and
$\overline K_t$ bounds the normalized scalar residual's Lipschitz constant.
The proof is Eq.~\eqref{eq:transfer-chart} followed by the coupling argument
in Theorem~\ref{thm:interface-transport}.
Crucially, the edge gains $A_{vu}$ must still bound the complete physical
predictor: the chart, reference point, scale, and rotation also change when
parent predictions change.  For a $C^1$ path norm the reconstruction
derivatives must additionally be included; $d_{\max}$ alone is then
insufficient.

\subsubsection{Uniform trajectory error and normalized point clouds}

The learning experiment observes normalized curve samples, whereas the
composition bound first controls a trajectory at common driving phases.
The following results connect these descriptions in two steps: control the
effect of centering and scaling, then the effect of sampling each curve by
its own arc length. A lower bound on curve scale prevents normalization
from magnifying arbitrarily small positional errors.

\begin{corollary}[Uniform trajectory error]
\label{cor:interface-uniform}
Under Proposition~\ref{prop:transfer-propagation}, suppose
$\delta_t(z)\le e_t$ for every admissible true input in the evaluation
population.  Then
\begin{equation}
 \|\widehat y-y\|_{\mathsf Y}
 \le B:=\sum_v w_ve_{t(v)}
 \label{eq:interface-uniform}
\end{equation}
almost surely.  In particular, for $\mathsf Y=C([0,1];\mathbb R^2)$ with
the supremum norm this bounds the error at every common driving phase.
\end{corollary}

\begin{proof}
Substitute the uniform local bounds in
Eq.~\eqref{eq:transfer-pointwise}.
\end{proof}

For an ordered cloud $P=(p_1,\ldots,p_N)$ in $\mathbb R^2$, write
\[
 \|P\|_N^2=\frac1N\sum_{i=1}^N\|p_i\|_2^2,
 \qquad
 CP=(p_i-\overline p)_{i=1}^N,
 \qquad
 \sigma(P)=\|CP\|_N,
\]
and define the centered, RMS-normalized cloud
$\mathcal N(P)=CP/\sigma(P)$ when $\sigma(P)>0$.
We use the following explicit symmetric squared Chamfer convention:
\begin{equation}
 \operatorname{Ch}_2(P,Q)=\frac1{2N}
 \left(\sum_{i=1}^N\min_j\|p_i-q_j\|_2^2+
       \sum_{j=1}^N\min_i\|q_j-p_i\|_2^2\right).
 \label{eq:transfer-chamfer}
\end{equation}
For the convention that sums the two directional means, the following
Chamfer upper bounds are multiplied by two.

\begin{proposition}[Stable normalization and Chamfer comparison]
\label{prop:transfer-normalization}
If $\sigma(P),\sigma(Q)\ge s_0>0$, then
\begin{equation}
 \operatorname{Ch}_2\bigl(\mathcal N(P),\mathcal N(Q)\bigr)
 \le\|\mathcal N(P)-\mathcal N(Q)\|_N^2
 \le\frac{\|P-Q\|_N^2}{s_0^2}.
 \label{eq:transfer-normalized-chamfer}
\end{equation}
Consequently, sampling true and predicted trajectories at the same $N$
phases, under Corollary~\ref{cor:interface-uniform} and these scale floors,
gives a normalized Chamfer bound $B^2/s_0^2$.
If only $\sigma(P)\ge s_0$ is known and $\|P-Q\|_N\le B<s_0$, then $Q$
is also nonconstant and the bound
$4B^2/s_0^2$ holds instead.
\end{proposition}

\begin{proof}
The centering map $C$ is an orthogonal projection for the inner product
$\langle P,Q\rangle_N=N^{-1}\sum_i p_i^\top q_i$, and hence
$\|C(P-Q)\|_N\le\|P-Q\|_N$.  For nonzero vectors $a,b$ in any real
Hilbert space, put $r=\|a\|$ and $s=\|b\|$.  Direct expansion gives
\[
 \|a-b\|^2=(r-s)^2+
 rs\left\|\frac a r-\frac b s\right\|^2.
\]
Apply this identity with $a=CP$, $b=CQ$ and $r,s\ge s_0$ to obtain the
second inequality in Eq.~\eqref{eq:transfer-normalized-chamfer}.
In each directed Chamfer term the point with the same index is an
available match, proving the first inequality.  Shared-phase sampling
satisfies $\|P-Q\|_N\le\|y-\widehat y\|_\infty\le B$.

For the final assertion, the reverse triangle inequality gives
$\sigma(Q)\ge\sigma(P)-\|C(P-Q)\|_N\ge s_0-B>0$.
For arbitrary nonzero $a,b$,
\[
 \left\|\frac a{\|a\|}-\frac b{\|b\|}\right\|
 \le\frac{\|a-b\|+|\|a\|-\|b\||}{\|a\|}
 \le\frac{2\|a-b\|}{\|a\|}.
\]
Use $a=CP$, the true scale floor, and the preceding Chamfer comparison.
\end{proof}

The same proof gives an expected normalized Chamfer bound
$B_{\mathrm{rms}}^2/s_0^2$ from any supremum-path RMS bound
$\bigl(\mathbb E\|y-\widehat y\|_\infty^2\bigr)^{1/2}
\le B_{\mathrm{rms}}$, provided both cloud scale floors hold almost surely.
This comparison uses points sampled at the same driving phases.
With independent arc-length sampling, corresponding samples can occur at
different phases. Controlling the derivative error also controls that
displacement in phase; a positive speed floor makes the arc-length maps
stably invertible. These are the additional conditions in the next result.

\begin{proposition}[An explicit arc-length resampling condition]
\label{prop:transfer-arclength}
Let $y,\widehat y\in C^1([0,1];\mathbb R^2)$ satisfy
\[
 0<a\le\|y'(\theta)\|_2,\|\widehat y'(\theta)\|_2\le b<\infty,
 \qquad
 \|y-\widehat y\|_\infty\le E_0,
 \quad
 \|y'-\widehat y'\|_\infty\le E_1.
\]
Let $s_y(\theta)$ and $s_{\widehat y}(\theta)$ be cumulative arc length
divided by the respective total lengths, each starting at phase zero.
For any common fractions $u_1,\ldots,u_N\in[0,1]$, set
\[
 p_i=y(s_y^{-1}(u_i)),\qquad
 q_i=\widehat y(s_{\widehat y}^{-1}(u_i)).
\]
Then
\begin{equation}
 \|P-Q\|_N\le E_0+2(b/a)^2E_1=:E_{\rm arc}.
 \label{eq:transfer-arclength-bound}
\end{equation}
If both sampled clouds have RMS scale at least $s_0$, their normalized
squared Chamfer distance is at most $E_{\rm arc}^2/s_0^2$.
\end{proposition}

\begin{proof}
Write $\ell_y(\theta)=\int_0^\theta\|y'(r)\|_2\,\mathrm dr$,
$L_y=\ell_y(1)$, and define their predicted counterparts analogously.
The reverse triangle inequality for speeds implies
\[
 \|\ell_y-\ell_{\widehat y}\|_\infty\le E_1,
 \qquad |L_y-L_{\widehat y}|\le E_1.
\]
Since $L_y,L_{\widehat y}\in[a,b]$ and
$\ell_{\widehat y}(\theta)\le L_{\widehat y}$,
\[
 \|s_y-s_{\widehat y}\|_\infty
 \le\frac{E_1}{L_y}
    +\frac{|L_y-L_{\widehat y}|}{L_y}
 \le\frac{2E_1}{a}.
\]
Both normalized arc-length maps have derivative at least $a/b$ and are
strictly increasing bijections of $[0,1]$.  Their inverses are therefore
$b/a$-Lipschitz.  Comparing their values at any common fraction $u$ gives
\[
 |s_y^{-1}(u)-s_{\widehat y}^{-1}(u)|
 \le\frac b a\|s_y-s_{\widehat y}\|_\infty
 \le\frac{2bE_1}{a^2}.
\]
Use the position error $E_0$ at the true inverse phase, and then the
$b$-Lipschitz continuity of $\widehat y$, to obtain
$\|p_i-q_i\|_2\le E_0+2b^2E_1/a^2$ for every $i$.
Take the cloud RMS norm and apply
Proposition~\ref{prop:transfer-normalization}.
\end{proof}

For example, the abstract propagation result may be instantiated in
$C^1$ path spaces with norm
$\|f\|_{C^1}=\max\{\|f\|_\infty,\|f'\|_\infty\}$, provided the local
domain, error, and gain assumptions hold in that norm.  A uniform output
bound $B$ then supplies $E_0,E_1\le B$ in
Proposition~\ref{prop:transfer-arclength}.  Its nonzero-speed condition
states explicitly the additional regularity used by arc-length inversion.

\paragraph{What the bound makes visible.}
The local error, distribution shift, and path weight answer three practical
questions: how accurately is a conditional rule learned, how familiar are
its inputs in a new assembly, and how much does its error matter to the
output? They identify where additional training coverage could be useful:
unfamiliar interface conditions with large downstream influence. The next
two constructions examine that influence directly. They show why small,
familiar module types can still form a sensitive mechanism.

%% file: appendices/theory/stability.tex

\subsection{Local regularity does not bound global configuration sensitivity}
\label{app:theory-dyad-sensitivity}

The transfer bound weights local errors by their downstream influence.
Can that influence be inferred from how well-conditioned each module is
in isolation? A dyad gives a concrete reason to look further. Its output
joint moves tangentially to the circle around its fixed anchor. The next
dyad receives that motion in a particular direction, which determines how
much of its potential amplification is realized.

The construction below arranges these directions in two ways. One chain
passes the crank's instantaneous displacement through unchanged; the other
doubles it at every stage. Both chains have the same graph, unit bar
lengths, and local Jacobian singular values. Even their local solution
derivatives have the same operator norm. The difference lies in how the
responses line up across interfaces.

Let $e_x=(1,0)^\top$ and $e_y=(0,1)^\top$.  A dyad with internal joint $q$,
moving boundary joint $p$, and fixed boundary joint $a$ has constraints
\begin{equation}
 f(q,p;a)=\frac12
 \begin{pmatrix}
 \lVert q-p\rVert^2-1\\
 \lVert q-a\rVert^2-1
 \end{pmatrix}=0.
 \label{eq:theory-unit-dyad-constraints}
\end{equation}
At a feasible configuration write $u=q-p$ and $v=q-a$, so that $u$ and $v$
are unit vectors.  The Jacobian with respect to the unknown joint is
\begin{equation}
 J=\partial_q f=
 \begin{pmatrix}u^\top\\v^\top\end{pmatrix}.
 \label{eq:theory-dyad-normalized-jacobian}
\end{equation}
The factor $1/2$ in~\eqref{eq:theory-unit-dyad-constraints} fixes the
constraint scaling: both rows of $J$ have unit Euclidean norm.  Throughout
this subsection, a local Jacobian means this normalized Jacobian.

\begin{theorem}[Equal local regularity, exponentially different global sensitivity]
\label{thm:theory-dyad-sensitivity-separation}
For every integer $m\geq 1$, there are two grounded planar bar--joint
mechanisms, denoted $\mathcal M_m^{\mathrm{constant}}$ and
$\mathcal M_m^{\mathrm{amplifying}}$, with the following properties.
\begin{enumerate}
 \item Both mechanisms have the same incidence graph: a unit crank
 drives $q_0$, and, for $i=1,\ldots,m$, a two-bar dyad determines $q_i$
 from $q_{i-1}$ and a fixed anchor $a_i$.  Every moving bar has unit
 length.  The anchors lie on one rigid ground body.  The two mechanisms
 have the same Assur solve graph and a chain of $m$ successive dyads.
 \item Both admit a regular analytic assembly branch through their
 reference configurations for $|\theta|<\delta_m$, for some
 $\delta_m>0$, under the same unit-speed-at-reference crank input
 \begin{equation}
 q_0(\theta)=(-1+\cos\theta,\sin\theta)^\top.
 \label{eq:theory-common-crank}
 \end{equation}
 \item At $\theta=0$, every dyad in both mechanisms has the same two
 local Jacobian singular values,
 \begin{equation}
 \sigma_-=\sqrt{1-\frac{2}{\sqrt5}},\qquad
 \sigma_+=\sqrt{1+\frac{2}{\sqrt5}}.
 \label{eq:theory-common-singular-values}
 \end{equation}
 The operator norm of each local solution derivative with respect to
 its moving boundary is also the same, namely $\sqrt5$.
 \item Their endpoint configuration sensitivities at that input differ
 exponentially:
 \begin{equation}
 \left\lVert\frac{d q_m^{\mathrm{constant}}}{d\theta}(0)\right\rVert=1,
 \qquad
 \left\lVert\frac{d q_m^{\mathrm{amplifying}}}{d\theta}(0)\right\rVert=2^m.
 \label{eq:theory-dyad-sensitivity-separation}
 \end{equation}
\end{enumerate}
\end{theorem}

\begin{proof}
We first derive the local transfer formula, then specify the two grounded
mechanisms explicitly.

\paragraph{The transfer through one dyad.}
Let $R(x,y)^\top=(-y,x)^\top$ denote counterclockwise rotation through
$\pi/2$, and put $t=Rv$.  Because $a$ is fixed, differentiating
\eqref{eq:theory-unit-dyad-constraints} gives
\begin{equation}
 u^\top(dq-dp)=0,\qquad v^\top dq=0.
 \label{eq:theory-dyad-differential}
\end{equation}
Thus $dq$ is a multiple of $t$.  Whenever $u^\top t\neq0$, solving these
two equations yields
\begin{equation}
 dq=A\,dp,\qquad
 A=\frac{t u^\top}{u^\top t}.
 \label{eq:theory-dyad-transfer}
\end{equation}
In particular, if the preceding joint moves with differential
$dp=t_{\mathrm{in}}\,ds$, the scalar gain into the outgoing unit tangent
$t$ is
\begin{equation}
 dq=t\,\frac{u^\top t_{\mathrm{in}}}{u^\top t}\,ds.
 \label{eq:theory-tangent-gain}
\end{equation}

\paragraph{A chain with constant endpoint sensitivity.}
At the reference configuration set $q_0^*=0$ and, for every
$i=1,\ldots,m$, define
\begin{equation}
 u_i=\frac{2e_x+e_y}{\sqrt5},\qquad
 v_i=e_x,\qquad t_i=e_y,\qquad
 q_i^*=q_{i-1}^*+u_i,\qquad a_i=q_i^*-v_i.
 \label{eq:theory-constant-chain}
\end{equation}
These equations specify both the joint coordinates and the fixed
anchors.  They give $\lVert q_i^*-q_{i-1}^*\rVert=
\lVert q_i^*-a_i\rVert=1$, as required.  The crank is grounded at
$a_0=-e_x$, so~\eqref{eq:theory-common-crank} passes through $q_0^*=0$
with derivative $q_0'(0)=e_y$.  In fact this chain has the explicit
assembly
\begin{equation}
 q_i(\theta)=q_i^*+q_0(\theta).
 \label{eq:theory-constant-chain-assembly}
\end{equation}
The interjoint bars retain vector $u_i$, while
$q_i(\theta)-a_i=(\cos\theta,\sin\theta)^\top$.  Equation
\eqref{eq:theory-constant-chain-assembly} therefore satisfies every bar
constraint and gives $q_i'(0)=e_y$ for all $i$.

\paragraph{A chain with amplifying endpoint sensitivity.}
Again set $q_0^*=0$.  For odd $i$, set
\begin{equation}
 t_i=e_x,\qquad v_i=-e_y,\qquad
 u_i=\frac{e_x+2e_y}{\sqrt5};
 \label{eq:theory-amplifying-odd}
\end{equation}
for even $i$, set
\begin{equation}
 t_i=e_y,\qquad v_i=e_x,\qquad
 u_i=\frac{2e_x+e_y}{\sqrt5}.
 \label{eq:theory-amplifying-even}
\end{equation}
Define the coordinates and anchors recursively by
\begin{equation}
 q_i^*=q_{i-1}^*+u_i,\qquad a_i=q_i^*-v_i.
 \label{eq:theory-amplifying-coordinates}
\end{equation}
Again all moving bars have unit length.  The internal joints advance
strictly in both coordinate directions.  Their explicit coordinates are
\begin{equation}
 q_{2r}^*=\frac{3r}{\sqrt5}(e_x+e_y),\qquad
 q_{2r+1}^*=\frac{(3r+1)e_x+(3r+2)e_y}{\sqrt5}.
 \label{eq:theory-amplifying-explicit-coordinates}
\end{equation}
Put $t_0=e_y$, the crank tangent at $\theta=0$.  At every step,
\begin{equation}
 u_i^\top t_i=\frac1{\sqrt5},\qquad
 u_i^\top t_{i-1}=\frac2{\sqrt5}.
 \label{eq:theory-amplifying-dot-products}
\end{equation}
The transfer formula~\eqref{eq:theory-tangent-gain} therefore gives,
inductively,
\begin{equation}
 q_i'(0)=2^i t_i.
 \label{eq:theory-amplifying-velocities}
\end{equation}

\paragraph{Grounding and the common constraint structure.}
In each mechanism all $a_i$, including $a_0$, are fixed points carried
by a single rigid ground body.  Ground-anchor positions may differ
between the two mechanisms; their ground-link incidence is identical.
All internal joints are distinct.  There are also no accidental
identifications between an anchor and an internal joint: internal-joint
coordinates lie in $(1/\sqrt5)\mathbb Z^2$, whereas the difference
$q_i^*-a_i$ is a nonzero coordinate unit vector.  Irrationality of
$\sqrt5$ excludes equality with any other internal joint.  The same
argument, together with distinct internal coordinates, distinguishes
the anchors from each other.  The two constraints of each dyad are
independent at the reference configuration, as verified below.  With
its two boundary joints pinned, a dyad is therefore an irreducible
isostatic module.  Each module uses only the preceding moving joint
and its ground anchor.  If the source group includes all ground pins,
the full solve graph also has source-to-module edges; these are the
same in both constructions.

\paragraph{Equal local singular values and transfer norms.}
For any of the constructed dyads,
\begin{equation}
 J_iJ_i^\top=
 \begin{pmatrix}
 1&u_i^\top v_i\\
 u_i^\top v_i&1
 \end{pmatrix},\qquad
 \left|u_i^\top v_i\right|=\frac2{\sqrt5}.
 \label{eq:theory-common-gram-matrix}
\end{equation}
Its eigenvalues are $1\pm2/\sqrt5$, proving
\eqref{eq:theory-common-singular-values}.  In particular,
$\det J_i\neq0$ and both reference assemblies are regular.  Since
$u_i$ and $t_i$ are unit vectors and $|u_i^\top t_i|=1/\sqrt5$,
the rank-one formula~\eqref{eq:theory-dyad-transfer} also gives
$\lVert A_i\rVert_2=\sqrt5$ for every module in both mechanisms.

\paragraph{Existence of the assembly branches.}
Substitute the analytic crank~\eqref{eq:theory-common-crank} into all
dyad constraints, and treat $(q_1,\ldots,q_m)$ as the unknown vector.
The Jacobian of the resulting system at the reference configuration
is block lower triangular, with invertible diagonal blocks
$J_1,\ldots,J_m$.  The analytic implicit function theorem gives a
unique analytic assembly branch through that configuration on an
open interval about $\theta=0$.  Continuity allows the interval to be
reduced so that every dyad remains regular and all distinct joints
remain distinct.  Taking the smaller interval for the two mechanisms
gives the stated $\delta_m>0$.  This is an existence statement for
each finite $m$; it does not require an interval of uniform width as
$m$ grows.  The derivatives computed above prove
\eqref{eq:theory-dyad-sensitivity-separation} and complete the proof.
\end{proof}

The local operator norm records the largest possible response over input
directions. Composition selects a particular direction at each stage.
Theorem~\ref{thm:theory-dyad-sensitivity-separation} shows that these choices
can make the endpoint sensitivity constant or exponential even when all
local norms agree. Geometric interface information therefore explains
differences that the solve graph and local conditioning alone leave open.
Here sensitivity refers to the configuration map
$\theta\mapsto q_m(\theta)$; its relation to a curve observation also depends
on the readout and sampling, as formalized in
Appendix~\ref{app:interface-transfer}.

%% file: appendices/theory/singular-depth.tex

\subsection{Repeated tangencies reduce configuration regularity}
\label{app:theory-singular-depth}

The preceding chains remain regular near their reference poses: their
responses have finite derivatives, although those derivatives can become
large with depth. At a circle tangency, the form of the response itself
changes. Moving two tangent circles slightly toward one another opens an
intersection height proportional to the square root of the displacement.
If that height supplies the displacement for the next dyad, two stages
produce a fourth-root response, three an eighth-root response, and so on.

The theorem realizes this repeated square-root effect with unit bars.
The motion remains continuous at the limiting pose, but becomes progressively
less smooth: a small change in the driving angle can produce a much larger
change in the final joint position. The exponent $2^{-m}$ states precisely
how this regularity deteriorates with the number of successive tangencies.

\begin{theorem}[Depth-dependent loss of regularity at a singular assembly]
\label{thm:theory-singular-depth}
For every integer $m\geq1$, there is a grounded planar mechanism formed
by a unit crank followed by $m$ unit-bar dyads, and a continuous assembly
branch $q_1(\theta),\ldots,q_m(\theta)$ for $0\leq\theta<\eta_m$, with
$\eta_m>0$, such that every dyad is regular for $0<\theta<\eta_m$ and
\begin{equation}
 \lim_{\theta\downarrow0}
 \frac{\lVert q_m(\theta)-q_m(0)\rVert}
 {\theta^{2^{-m}}}=1.
 \label{eq:theory-singular-depth-asymptotic}
\end{equation}
Consequently, the endpoint configuration map is not H\"older continuous
with exponent $\beta>2^{-m}$ at $\theta=0$.  In particular, over
mechanisms with arbitrarily many dyads, no positive depth-independent
H\"older exponent holds at all singular assemblies.
\end{theorem}

\begin{proof}
Let $n_k=e_x$ for odd $k$ and $n_k=e_y$ for even $k$.  In particular,
$n_k$ and $n_{k+1}$ are perpendicular unit vectors.  Put $q_0^*=0$ and
specify the reference positions and ground anchors by
\begin{equation}
 q_k^*=q_{k-1}^*+n_k,\qquad
 a_k=q_{k-1}^*+2n_k\qquad (1\leq k\leq m).
 \label{eq:theory-singular-reference}
\end{equation}
Each reference dyad has two unit bars meeting at an external tangency:
$q_k^*$ is the midpoint of $q_{k-1}^*$ and $a_k$.  The crank, grounded
at $a_0=e_y$, is
\begin{equation}
 q_0(\theta)=(\sin\theta,1-\cos\theta)^\top
 =q_0^*+\theta n_1+o(\theta).
 \label{eq:theory-singular-crank}
\end{equation}
All $a_k$ belong to one fixed ground body.  The reference joints satisfy
$q_{2r}^*=(r,r)^\top$ and $q_{2r+1}^*=(r+1,r)^\top$.  The odd anchors
have first coordinate minus second coordinate equal to $2$, and the
even anchors, including $a_0$, have this difference equal to $-1$.
Thus the internal joints and anchors are all distinct.

We construct the branch inductively and prove the more precise
asymptotic statement
\begin{equation}
 q_k(\theta)-q_k^*
 =\theta^{2^{-k}}n_{k+1}
 +o\bigl(\theta^{2^{-k}}\bigr),
 \qquad 0\leq k\leq m.
 \label{eq:theory-singular-vector-asymptotic}
\end{equation}
The case $k=0$ is~\eqref{eq:theory-singular-crank}.  Suppose the statement
holds for $k-1$, and write
\begin{equation}
 z=q_{k-1}(\theta)-q_{k-1}^*,\qquad
 \alpha=2^{-(k-1)}.
 \label{eq:theory-singular-induction-variables}
\end{equation}
Then $z=\theta^\alpha n_k+o(\theta^\alpha)$.  The vector between the
two boundary joints is
\begin{equation}
 r=a_k-q_{k-1}(\theta)=2n_k-z,
 \qquad d=\lVert r\rVert.
 \label{eq:theory-singular-boundary-distance}
\end{equation}
The squared height of a unit-circle intersection above the midpoint is
\begin{equation}
 h^2=1-\frac{d^2}{4}
 =n_k^\top z-\frac{\lVert z\rVert^2}{4}
 =\theta^\alpha\bigl(1+o(1)\bigr).
 \label{eq:theory-singular-height}
\end{equation}
For every sufficiently small positive $\theta$, this gives $h>0$ and
$0<d<2$, so the two circles intersect transversely.  Choose the unit
vector $t(r)$ perpendicular to $r$ whose limit as $\theta\downarrow0$
is $n_{k+1}$.  A continuous local choice is obtained by taking one of
the two fixed signs of a $90$-degree rotation of $r/d$.  Define the
chosen assembly by the circle-intersection formula
\begin{equation}
 q_k(\theta)=\frac{q_{k-1}(\theta)+a_k}{2}+h\,t(r).
 \label{eq:theory-singular-circle-intersection}
\end{equation}
This point is at unit distance from both boundaries.  Moreover,
\begin{align}
 q_k(\theta)-q_k^*
 &=\frac{z}{2}+h\,t(r)\nonumber\\
 &=\theta^{\alpha/2}n_{k+1}
   +o\bigl(\theta^{\alpha/2}\bigr),
 \label{eq:theory-singular-step}
\end{align}
because $z=O(\theta^\alpha)$, $h/\theta^{\alpha/2}\to1$, and
$t(r)\to n_{k+1}$.  This proves the induction.  Each stage permits a
sufficiently small positive interval; their intersection is nonempty
for any fixed finite $m$.  On a possibly smaller such interval all
distinct joints remain distinct.  The two constraint-gradient rows
of a dyad have determinant of magnitude $dh>0$, so every module is
regular there.  The formula is continuous at $\theta=0$ with limit
$q_k^*$.

Taking norms in~\eqref{eq:theory-singular-vector-asymptotic} proves
\eqref{eq:theory-singular-depth-asymptotic}.  If a H\"older bound of
exponent $\beta>2^{-m}$ held at zero, there would be a finite constant
$C$ with
$\lVert q_m(\theta)-q_m(0)\rVert\leq C\theta^\beta$
for all sufficiently small positive $\theta$.  Equation
\eqref{eq:theory-singular-depth-asymptotic} contradicts this, since
$\theta^{2^{-m}-\beta}\to\infty$.  Finally, for any fixed $\beta>0$
one can choose $m$ with $2^{-m}<\beta$, proving the last assertion.
\end{proof}

Together, the two chain constructions distinguish amplification on regular
branches from loss of smoothness at singular limits. Both arise with the
simplest Assur modules. Thus identifying reusable module types is only one
part of characterizing a prediction problem; the geometric regimes in which
those modules operate determine the regularity available to a learner.
We next turn from sensitivity to information: what happens when the
representation fails to distinguish two such conditional problems at all?

%% file: appendices/theory/information.tex

\subsection{Why boundary-port assignments matter}
\label{app:port-information}

Sensitivity concerns how prediction errors propagate when the conditional
problem is specified. A more basic obstruction occurs when two different
problems receive the same representation. Consider a dyad attached to a
moving crank joint and a fixed ground joint, with bar lengths $2$ and $3$.
An unordered list of lengths records the same module in both cases, but
does not say which length connects to the moving boundary. Swapping that
assignment changes the resulting motion.

The construction below makes the difference visible even after removing
position, orientation, scale, and traversal order. Both output traces are
circular arcs, but their central angles differ: after equalizing their
chords, one arc bends more than the other. If the representation merges
these mechanisms, every predictor must give them the same answer. The
separation between the true arcs therefore forces a nonzero error on at
least one of them, regardless of the predictor's capacity or training data.

\begin{proposition}[A port-assignment obstruction]
\label{prop:port-information}
Let $O=(0,0)$ and $B=(3,0)$ be fixed joints and let
$A(\theta)=(\cos\theta,\sin\theta)$, $0\leq\theta\leq2\pi$, be the end of a
unit input crank about $O$.  Attach a dyad with internal joint $C$ by
\begin{equation}
 \|C-A(\theta)\|=a,\qquad \|C-B\|=b,
 \label{eq:port-example-constraints}
\end{equation}
and select the branch satisfying
$\det(B-A(\theta),C-A(\theta))>0$.
Consider the two assignments $(a,b)=(3,2)$ and $(a,b)=(2,3)$.
They have identical module types, solve dependencies, boundary motions,
branch signs, and unordered dyad-length multisets.  Both have a regular
solution throughout a complete crank rotation.  Their output curves are
circular arcs with central angles
\begin{align}
 \psi_1&=\arccos(-3/4)-\arccos(1/4),\nonumber\\
 \psi_2&=\arccos(-17/18)-\arccos(-1/2),
 \label{eq:port-example-angles}
\end{align}
where $0<\psi_2<\psi_1<\pi$.  Consequently, their output curves remain
distinct after quotienting by Euclidean similarities, reflection, and
traversal order.

Let $\overline{\mathcal C}_i$ denote these arcs after their chord endpoints
are mapped to $(0,0)$ and $(1,0)$, with the arc above the chord, and put
\begin{equation}
 s_i=\frac12\tan(\psi_i/4),\qquad
 \Delta=s_1-s_2>0.
 \label{eq:port-example-gap}
\end{equation}
If a representation $\phi$ assigns the same features to the two mechanisms,
every deterministic predictor $g$ of nonempty compact output curves
satisfies
\begin{equation}
 \max_{i\in\{1,2\}}
 d_H\!\left(g(\phi(M_i)),\overline{\mathcal C}_i\right)
 \ \geq\ \frac{\Delta}{2}
 \ \approx\ 0.0255918,
 \label{eq:port-example-floor}
\end{equation}
where $d_H$ is Hausdorff distance.  The corresponding worst-case absolute
error for predicting the similarity-invariant sagitta-to-chord ratio $s_i$
has the same lower bound.
\end{proposition}

\begin{proof}
Write $D(\theta)=\|B-A(\theta)\|$.  Since
$D(\theta)^2=10-6\cos\theta$, we have $2\leq D(\theta)\leq4$.
For both assignments,
\begin{equation}
 |a-b|=1<D(\theta)<5=a+b.
 \label{eq:port-example-regularity}
\end{equation}
The two defining circles therefore intersect transversely at every input
angle.  More explicitly, set
\[
 e=\frac{B-A}{D},\qquad
 t=\frac{a^2+D^2-b^2}{2D},\qquad
 h=\sqrt{a^2-t^2}>0,
\]
and let $e^\perp$ be the counterclockwise quarter-turn of $e$.
The selected solution is
$C=A+te+he^\perp$, a smooth $2\pi$-periodic function of $\theta$.
For the two squared-distance constraints in
Eq.~\eqref{eq:port-example-constraints}, the Jacobian with respect to $C$
has determinant
$4\det(C-A,C-B)=4Dh>0$.  Thus the local position solve is nonsingular
throughout the complete rotation.

The selected branch lies strictly above the horizontal axis.
To see this, a solution with $C_y=0$ would have $C_x=3-b$ or $3+b$.
For $(a,b)=(3,2)$ these are $1$ and $5$, whose distances to points of the
unit circle lie in $[0,2]$ and $[4,6]$, respectively, and hence never equal
$a=3$.  For $(a,b)=(2,3)$ they are $0$ and $6$, with distance ranges
$\{1\}$ and $[5,7]$, neither containing $a=2$.
Neither branch can cross the horizontal axis.  At $\theta=0$ the selected
branch has positive $y$ coordinate, and the other branch has negative
$y$ coordinate; continuity establishes these signs for the whole cycle.

Write a point on the upper half of the output circle as
\[
 C=B+b(\cos\varphi,\sin\varphi),\qquad 0<\varphi<\pi.
\]
Such a point is reached for some input angle if and only if the circle
centred at $C$ with radius $a$ intersects the unit circle centred at $O$.
Because $a>1$, this is precisely
\[
 a-1\leq\|C\|\leq a+1.
\]
Every such intersection belongs to the selected branch, since its $C_y$
coordinate is positive.  This proves that the complete traced set, including
its endpoints, is exactly the arc specified by
\begin{equation}
 \frac{(a-1)^2-9-b^2}{6b}
 \ \leq\ \cos\varphi\ \leq\
 \frac{(a+1)^2-9-b^2}{6b}.
 \label{eq:port-example-full-arc}
\end{equation}
Substituting the two assignments gives Eq.~\eqref{eq:port-example-angles}.
Both angles lie strictly between zero and $\pi$, and
\[
 \cos\psi_1=\frac{-3+\sqrt{105}}{16},\qquad
 \cos\psi_2=\frac{17+\sqrt{105}}{36}.
\]
Their difference satisfies
$\cos\psi_2-\cos\psi_1=(95-5\sqrt{105})/144>0$,
so $\psi_1>\psi_2$.  A Euclidean similarity, including a reflection,
preserves the central angle of a proper circular arc.  Traversal changes
also leave the traced set unchanged.  The two curves are therefore
different under all of the stated identifications.

An arc with radius $R$ and central angle $\psi<\pi$ has chord length
$2R\sin(\psi/2)$ and sagitta $R(1-\cos(\psi/2))$.
Their ratio is $\tfrac12\tan(\psi/4)$, which establishes
Eq.~\eqref{eq:port-example-gap}.  Numerically,
\[
 \psi_1\approx1.1007423341,\quad \psi_2\approx0.7123013928,
 \qquad
 s_1\approx0.1411744917,\quad s_2\approx0.0899909200.
\]
For nonempty compact sets $K,L\subset\mathbb R^2$, the height functional
satisfies
$|\max_{x\in K}x_y-\max_{x\in L}x_y|\leq d_H(K,L)$.
Hence
$d_H(\overline{\mathcal C}_1,\overline{\mathcal C}_2)\geq\Delta$.
If their features coincide, both mechanisms receive the same prediction
$K=g(\phi(M_1))=g(\phi(M_2))$.
The Hausdorff triangle inequality then gives
\[
 \Delta\leq
 d_H(\overline{\mathcal C}_1,K)+d_H(K,\overline{\mathcal C}_2),
\]
which proves Eq.~\eqref{eq:port-example-floor}.
The scalar statement follows from the same triangle inequality on
$\mathbb R$.
\end{proof}

The numerical floor is expressed for the specified chord normalization and
Hausdorff loss.  The preceding non-equivalence of the two shapes holds under
any faithful normalization by similarities and traversal order.  The
proposition concerns representations that discard port assignments; a
record retaining complete link incidence distinguishes the two mechanisms.

The observation in the experiment is a finite point cloud. To connect the
construction to that target, the next proposition shows that the difference
in arc angle remains visible after sampling. The obstruction does not
depend on resolving arbitrarily fine details of a continuous curve.

\begin{proposition}[The distinction survives finite curve sampling]
\label{prop:port-information-sampling}
For either mechanism in Proposition~\ref{prop:port-information}, sample one
complete crank cycle at $m$ equally spaced periodic arc-length positions,
with arbitrary starting phase.  The sampled points occupy a circular arc
whose minimal covering central angle $\widehat\psi_i$ satisfies
\begin{equation}
 \psi_i(1-2/m)\leq\widehat\psi_i\leq\psi_i.
 \label{eq:sampled-arc-angle}
\end{equation}
For every $m\geq6$, the two sampled point sets are not related by a Euclidean
similarity, including reflection.  This includes $128$ periodic samples
and $128$ samples with the final point repeating the initial point.
\end{proposition}

\begin{proof}
Each point in the interior of the arc in
Eq.~\eqref{eq:port-example-full-arc} admits exactly two input-crank
positions: the circle centred at $C$ with radius $a$ intersects the unit
circle in two distinct points.  Each endpoint admits exactly one, where
those circles are tangent.  All these configurations are on the selected
dyad branch, as established in Proposition~\ref{prop:port-information}.
The crank therefore traverses the output arc once in each direction.
Equivalently, the derivative of the output angle can vanish only when
$(C-A)$ is parallel to $A$, which is precisely one of those two tangencies;
the denominator of implicit differentiation is nonzero by dyad regularity.
The total output-path length is consequently $2b\psi_i$.

Every arc-length position on this closed traversal is within cyclic
arc-length distance $b\psi_i/m$ of a sample.
In particular, each turning point has a sampled point within angular
distance $\psi_i/m$.  All samples lie between the two turning points,
which proves Eq.~\eqref{eq:sampled-arc-angle}.
For $m\geq6$, there are at least three distinct sampled points, since
each interior output point has at most two preimages on the traversal.
They determine the output circle, and the shortest containing arc has
angle $\widehat\psi_i<\pi$.  This angle is invariant under similarities
and reflection.  Finally,
\[
 \widehat\psi_1\geq\tfrac23\psi_1
 \approx0.7338282>0.7123014\approx\psi_2
 \geq\widehat\psi_2.
\]
Thus the sampled shapes are distinct under the stated identifications.
With $128$ samples including a duplicate closing point, the periodic set
has $m=127$ samples, so the same conclusion applies.
\end{proof}

\subsection{The dimension required to retain observable behaviour}
\label{app:observable-dimension}

The port example gives one explicit collision between different behaviours.
The next result asks when such collisions become unavoidable for an entire
family. The relevant size is the number of independently observable
variations: normalization may remove some parameter changes, while others
remain visible in the target. A continuous representation with too few
coordinates must merge some of these visible variations.

The proof uses a geometric fact: on the boundary of a $q$-dimensional
parameter ball, a continuous map into fewer than $q$ coordinates identifies
some opposite points. If their true outputs remain separated, a shared
prediction must incur an error. The theorem turns this idea into a
quantitative lower bound. Let
$B_r^q=\{\theta\in\mathbb R^q:\|\theta\|\leq r\}$ and let
$(\mathcal Y,d_{\mathcal Y})$ be a metric space of outputs after the
chosen normalization. The dimension argument uses the Borsuk--Ulam theorem; related latent-dimension obstructions for invariant representations are developed by \citet{Wagstaff2022}.

\begin{theorem}[An observable-dimension obstruction]
\label{thm:observable-dimension}
Let $T:B_r^q\to\mathcal Y$ describe a mechanism family with $r>0$.
Suppose there is a constant $c>0$ such that
\begin{equation}
 d_{\mathcal Y}(T(\theta),T(\theta'))
 \geq c\|\theta-\theta'\|
 \qquad\text{for every }\theta,\theta'\in B_r^q.
 \label{eq:observable-separation}
\end{equation}
If $d<q$, every continuous representation
$\phi:B_r^q\to\mathbb R^d$ and every decoder
$g:\mathbb R^d\to\mathcal Y$ satisfy
\begin{equation}
 \sup_{\theta\in B_r^q}
 d_{\mathcal Y}(g(\phi(\theta)),T(\theta))\geq cr.
 \label{eq:observable-dimension-floor}
\end{equation}
In particular, exact prediction of this family by continuous features
requires at least $q$ feature coordinates.
\end{theorem}

\begin{proof}
Restrict $\phi$ to the boundary sphere $rS^{q-1}$.
For $q\geq2$, pad its values with zeros to obtain a continuous map to
$\mathbb R^{q-1}$.  The Borsuk--Ulam theorem gives a point
$\theta\in rS^{q-1}$ with $\phi(\theta)=\phi(-\theta)$.
For $q=1$ the only admissible $d$ is zero and the same equality holds
directly.  Both parameters receive the same prediction $y$.
By Eq.~\eqref{eq:observable-separation} and the metric triangle inequality,
\[
 2cr\leq d_{\mathcal Y}(T(\theta),T(-\theta))
 \leq d_{\mathcal Y}(T(\theta),y)
      +d_{\mathcal Y}(y,T(-\theta)).
\]
At least one of the two errors is at least $cr$.
\end{proof}

How can the separation assumption be checked for a concrete mechanism
family? Measure several quantities from its output and differentiate them
with respect to the family parameters. Full column rank means that every
local parameter direction changes those measurements to first order.
The next proposition turns this observable rank into the separation
required by the theorem on a sufficiently small neighbourhood.

\begin{proposition}[An observable-rank criterion]
\label{prop:observable-rank}
Let $U\subset\mathbb R^q$ be open, $\theta_0\in U$, and
$T:U\to\mathcal Y$.
Suppose $a:\mathcal Y\to\mathbb R^p$ is $L$-Lipschitz with $L>0$,
$F=a\circ T$ is continuously differentiable near $\theta_0$, and
$DF(\theta_0)$ has column rank $q$.
Write $\sigma>0$ for its smallest singular value.
There exists $r>0$ such that, on
$\theta_0+B_r^q$, condition~\eqref{eq:observable-separation} holds with
$c=\sigma/(2L)$.  Consequently, every continuous representation of
dimension $d<q$ has worst-case prediction error at least
$\sigma r/(2L)$ on this neighbourhood.
\end{proposition}

\begin{proof}
Set $J=DF(\theta_0)$.
By continuity choose $r>0$ with $\theta_0+B_r^q\subset U$ and
\[
 \sup_{\theta\in\theta_0+B_r^q}\|DF(\theta)-J\|_{\mathrm{op}}
 \leq\sigma/2.
\]
For $\theta,\theta'$ in this convex ball, the fundamental theorem of
calculus along the segment between them gives
\[
 F(\theta)-F(\theta')
 =J(\theta-\theta')+
 \int_0^1\bigl(DF(\theta'+t(\theta-\theta'))-J\bigr)
             (\theta-\theta')\,dt.
\]
Therefore
$\|F(\theta)-F(\theta')\|\geq(\sigma/2)\|\theta-\theta'\|$.
The $L$-Lipschitz property of $a$ gives
\[
 d_{\mathcal Y}(T(\theta),T(\theta'))
 \geq\frac{\|F(\theta)-F(\theta')\|}{L}
 \geq\frac{\sigma}{2L}\|\theta-\theta'\|.
\]
Apply Theorem~\ref{thm:observable-dimension} after translating the
parameter ball by $\theta_0$.
\end{proof}

\paragraph{Finite observations.}
For the paper's target of $128$ sampled two-dimensional curve points,
a regular local coordinate description has at most $256$ continuous
coordinates before normalization; any independently removed continuous
symmetries further reduce its dimension.  Thus the observable rank $q$ in
Proposition~\ref{prop:observable-rank} is bounded by the dimension of that
chosen observation.  A dimension bound increasing without limit with
mechanism size would require increasingly detailed observations as well.

\subsection{Lossless structure exposes computation without adding information}
\label{app:lossless-information}

The preceding obstructions concern representations that merge outputs the
task must distinguish. A lossless structural decomposition raises a different
question: if the complete mechanism can be recovered, what has changed for
learning? The proposition below isolates the answer. An unrestricted learner
can translate between the raw and decomposed records, so both permit the
same predictions. The benefit of an explicit decomposition must therefore
be understood through the computation and inductive bias available to the
chosen learner.

\begin{proposition}[Statistical equivalence of lossless representations]
\label{prop:lossless-equivalence}
Let $E:\mathcal R\to\mathcal A$ be a measurable encoding of raw mechanism
records, and let $P:E(\mathcal R)\to\mathcal R$ be a measurable recovery
map satisfying $P(E(R))=R$.
For every sample size, every loss, and every family of joint laws of
$(R,Y)$, the sets of achievable prediction risks from $R$ and from $E(R)$
coincide when learners have unrestricted computation and may use $E$ and $P$.
Their optimal and minimax risks therefore coincide.
\end{proposition}

\begin{proof}
A learner receiving raw training records applies $E$ to every training
input and to the test input, then runs any learner for the encoded
representation.  Conversely, a learner receiving encoded records applies
$P$ to the training and test inputs and runs any raw-record learner.
These simulations give exactly the same predictions for every data set
and test input; randomized learners can use the same random seed.
Consequently, both simulations preserve risk under every joint law.
Taking infima over learners, and then suprema over any specified family
of laws, proves the assertions.
\end{proof}

This is the reason to compare representations with the same learner in the
paper's prediction experiment: it asks whether making constraint
organization explicit helps that learner use the available data. The
lossless Assur record and the fixed feature vector extracted from it are
distinct objects. The former exposes a recoverable computational structure;
the latter retains the interface and geometric details chosen by feature
extraction. The port and dimension results explain why those choices matter.

%% file: appendices/theory/structure.tex
\subsection{Constraint modules and graph width measure different structure}
\label{app:theory-structure}

The preceding results make the conditional solve structure central to reuse.
Could an ordinary measure of undirected graph complexity already capture
that structure? The size of an Assur component measures how many joint
positions must be determined together in the Assur decomposition. Treewidth
measures the size of overlapping vertex sets needed in an undirected tree
decomposition. The following two constructions show why they differ.

In the first, a grid is built one dyad at a time: each new joint attaches to
two joints whose positions are already known. The undirected grid becomes
wide, while each conditional solve remains small. In the second, joints
form a cycle and each also attaches to one of two fixed pins. The graph has
bounded treewidth, but no proper set of internal joints is determined from
the pins alone; the whole mechanism forms one Assur component. Both families
use only two boundary joints per component.

Throughout this subsection, $G=(I\mathbin{\dot\cup}P,E)$ is a simple pinned
bar--joint graph in the plane: $I$ contains the internal joints, $P$ contains
the fixed boundary joints, and every edge has an internal endpoint. We take
treewidth on this graph, including the pins and without adding edges between
them. ``In the plane'' refers to the coordinates of the joints; crossings
between bars are permitted and are not additional joints. For a generically
pinned-isostatic graph, let $\kappa(G)$ be the largest number of internal
joints in an Assur component and let $p(G)$ be the largest number of distinct
boundary joints of a component. Component boundaries may include joints
determined by earlier components.

\begin{theorem}[Separation of Assur component size and treewidth]
\label{thm:assur-treewidth-separation}
There are two families of generically pinned-isostatic planar bar--joint
graphs, each with exactly two pins, with the following properties:
\begin{enumerate}
  \item For every integer $r\geq2$, a graph $G_r$ with $(r+1)^2$ internal
  joints satisfies
  \[
     \kappa(G_r)=1,\qquad p(G_r)=2,\qquad
     \operatorname{tw}(G_r)\geq r.
  \]
  \item For every integer $n\geq3$, a graph $H_n$ with $n$ internal joints
  satisfies
  \[
     \kappa(H_n)=n,\qquad p(H_n)=2,\qquad
     \operatorname{tw}(H_n)\leq4.
  \]
\end{enumerate}
Consequently, bounded Assur component size does not imply bounded treewidth;
bounded treewidth and bounded interface arity together do not imply bounded
Assur component size.
\end{theorem}

\begin{proof}
\textbf{Dyadic graphs containing large grids.}
Set $s=r+1$, take two distinct pins $a,b$, and introduce joints
$q_{ij}$ for $1\leq i,j\leq s$ in row-major order. Attach each new joint to
two previously available joints as follows:
\[
\begin{array}{c|c}
\text{new joint} & \text{attachment joints}\\ \hline
q_{11} & a,b\\
q_{1j},\ j>1 & q_{1,j-1},a\\
q_{i1},\ i>1 & q_{i-1,1},b\\
q_{ij},\ i,j>1 & q_{i-1,j},q_{i,j-1}.
\end{array}
\]
There are no other edges. The two attachment joints are always distinct,
and an edge is introduced only with its later endpoint, so the graph is
simple. Every addition is a dyad with one internal joint and two boundary
joints. At a generic placement, its two incident bar directions are
linearly independent. Ordering the constraint rows and internal coordinate
columns by insertion gives a block lower-triangular pinned rigidity matrix
with invertible $2\times2$ diagonal blocks. Thus the graph is generically
pinned-isostatic. Its successive dyads form an Assur decomposition, so
$\kappa(G_r)=1$ and $p(G_r)=2$; equivalently, the Assur decomposition theorem
identifies these singleton components
\citep{ServatiusShaiWhiteley2010,ShaiSljokaWhiteley2013}.

The graph induced by the $q_{ij}$ contains the $s\times s$ grid, including
every horizontal and vertical grid edge. We give a direct lower bound on
its treewidth. For each row $i$ and column $j$, let $C_{ij}$ be the union of
that row and that column. Every $C_{ij}$ is connected and any two such
crosses intersect. In any tree decomposition, the bags meeting a connected
vertex set form a subtree. The subtrees corresponding to the crosses
therefore intersect pairwise. The Helly property of subtrees of a tree
gives a bag meeting every cross. Any set of fewer than $s$ grid vertices
misses an entire row and an entire column, and hence misses their cross.
That bag must consequently contain at least $s$ grid vertices. Its width is
at least $s-1=r$. Restricting a tree decomposition to a subgraph cannot
increase its width, so $\operatorname{tw}(G_r)\geq r$.

\textbf{Large irreducible components of bounded treewidth.}
Take internal vertices $v_1,\ldots,v_n$ forming a simple cycle and two pins
$a,b$. Give each internal vertex exactly one additional edge, to $a$ when
its index is odd and to $b$ when its index is even. Both pins are used for
every $n\geq3$. There are $n$ cycle edges and $n$ spokes, hence
$|E(H_n)|=2n$.

We verify the pinned-Laman counts, which characterize generic pinned
isostaticity in the plane \citep{ServatiusShaiWhiteley2010}. For a nonempty
internal vertex set $S$ and a selected pin set $P'$, the relevant induced
edge bounds are $2|S|$ when $|P'|\geq2$, $2|S|-1$ when $|P'|=1$, and
$2|S|-3$ when $P'=\emptyset$ and $|S|\geq2$. A single unpinned vertex has no
edges and is treated separately.

If $\emptyset\neq S\subsetneq I$, the induced cycle edges form a forest
and number at most $|S|-1$. Including any selected pins adds at most
$|S|$ spokes. Thus a subset with either one or two pins has at most
$2|S|-1$ edges, satisfying the corresponding bound. With no pins and
$|S|\geq2$, there are at most $|S|-1\leq2|S|-3$ edges. If $S=I$, the
no-pin subgraph has $n\leq2n-3$ edges, including equality at $n=3$.
With exactly one pin it has at most $n+(n-1)=2n-1$ edges because both pins
are used. With both pins it has exactly $2n$ edges. Induced edge counts
also bound all non-induced subgraphs, completing the count verification.

It follows that $H_n$ is generically pinned-isostatic. Moreover, every
proper nonempty internal subset has strictly fewer than $2|S|$ constraints
even with both pins available. Such a subset cannot be pinned-rigid: its
pinned rigidity matrix has $2|S|$ columns and fewer rows. There is therefore
no proper pinned-isostatic subassembly. The whole graph is an Assur
component, so $\kappa(H_n)=n$ and $p(H_n)=2$.

Finally, the bags
\[
   B_j=\{a,b,v_1,v_j,v_{j+1}\},\qquad j=2,\ldots,n-1,
\]
ordered along a path form a tree decomposition. They cover the first and
last cycle edges through $v_1$, every other cycle edge through its
corresponding bag, and every spoke through its internal endpoint. Each
vertex occurs in a connected interval of bags. All bags have five vertices,
so the width is at most four. For $n=3$, the same construction consists of
a single five-vertex bag and remains valid.
\end{proof}

\begin{corollary}[Linear conditional evaluation despite unbounded treewidth]
\label{cor:dyadic-evaluation-work}
For the family $G_r$ in Theorem~\ref{thm:assur-treewidth-separation}, suppose
the pin positions, all bar lengths, and a feasible assembly branch are
given, and every dyad on that branch has two distinct attachment points and
transverse circle intersections. All $(r+1)^2$ internal positions can then
be evaluated in $O(r^2)$ scalar operations in a model assigning unit cost
to arithmetic operations and square roots.
\end{corollary}

\begin{proof}
Let a dyad's attachment positions be $u,v$, its bar lengths be
$\ell_u,\ell_v$, and its specified branch be $\sigma\in\{-1,1\}$.
Writing $d=\|v-u\|$, $e=(v-u)/d$, and $e^\perp$ for a fixed
quarter-turn of $e$, its internal position is
\[
 q=u+\alpha e+\sigma h e^\perp,\qquad
 \alpha=\frac{\ell_u^2-\ell_v^2+d^2}{2d},\qquad
 h=\sqrt{\ell_u^2-\alpha^2}.
\]
The hypotheses give $d>0$ and $h>0$, so this uses a constant number of
operations. Row-major traversal evaluates each of the $(r+1)^2$ dyads
once, with its attachment positions already available.
\end{proof}

The grid gives a concrete computational consequence: a known feasible
branch is evaluated by one circle intersection per joint, even as its
treewidth grows. The cycle shows the converse distinction: small interfaces
and small treewidth can coexist with a single large Assur component.
Thus undirected connectivity and conditional solve structure describe
different aspects of a mechanism's complexity. The corollary counts scalar
operations for evaluating a supplied feasible branch at one boundary
configuration; numerical bit complexity and finding that branch are separate
computational questions.

%% file: appendices/K_kb_census.tex
\section{Census of the Built Knowledge Base}
\label{app:kb-census}

Here we report a complete census of the Assur-KB build present in this checkout. The census describes the dataset. The topology name and the source directory are summarised at a top-$50$ cutoff with the tail collapsed into a single catch-all bucket.

Throughout, \emph{module} means a candidate Assur group and \emph{order} is the stored interface arity of \S\ref{sec:reconstructability}, the number of distinct boundary joints. Group counts labelled \emph{total} include the ground group of each mechanism; \emph{Assur groups} exclude it.

\subsection{Corpus totals}

\input{tables/corpora_comparison}

\begin{table}[h]
\caption{Build totals per source. Extraction time is the decomposition cost
recorded on each record, summed over the source.}
\label{tab:kb-census-totals}
\centering
\small
\begin{tabular}{lrrr}
\toprule
Quantity & LINKS & Purwar Lab & Total \\
\midrule
Shards & 101 & 304 & 405 \\
Records & 10,020,000 & 2,966,754 & 12,986,754 \\
Unparsable lines & 0 & 0 & 0 \\
JSONL size (GB) & 131.1 & 26.7 & 157.8 \\
Mean extraction time (ms/record) & 1.64 & 4.17 & --- \\
Total extraction time (CPU h) & 4.57 & 3.43 & 8.00 \\
\bottomrule
\end{tabular}
\end{table}

The LINKS side is the \href{https://huggingface.co/datasets/ahn1376/LInK-10M}{\texttt{LInK-10M}} release, split $10{,}000{,}000$ train and $20{,}000$ test, all revolute, all decomposed by the dyadic route. The Purwar Lab side is streamed from the upstream \href{https://www.kaggle.com/datasets/purwarlab/four-six-and-eight-bar-mechanisms-with-curves}{\texttt{purwarlab}} release and decomposed by the link-incidence route; its records divide into $602{,}079$ four-bar, $397{,}525$ six-bar, and $1{,}967{,}150$ eight-bar mechanisms. Extraction per record: eight CPU hours cover the whole build, and the tail is short (the slowest LINKS record takes $8.2$~s and the slowest Purwar Lab record $4.8$~s). 

\subsection{Mechanism-level structure}

\begin{table}[h]
\caption{Per-mechanism distributions. Groups include the ground group;
DAG depth is the longest path in the condensation DAG.}
\label{tab:kb-census-mechanism}
\centering
\small
\begin{tabular}{lrrrrrrrr}
\toprule
& \multicolumn{4}{c}{LINKS} & \multicolumn{4}{c}{Purwar Lab} \\
\cmidrule(lr){2-5}\cmidrule(lr){6-9}
Quantity & min & med & mean & max & min & med & mean & max \\
\midrule
Joints & 6 & 13 & 12.91 & 20 & 5 & 11 & 9.49 & 11 \\
Links & 7 & 19 & 19.47 & 35 & 2 & 8 & 6.82 & 8 \\
Ground joints & 2 & 3 & 2.68 & 7 & 2 & 3 & 2.79 & 4 \\
Groups (total) & 4 & 10 & 10.23 & 18 & 1 & 4 & 3.57 & 5 \\
Assur groups & 3 & 9 & 9.23 & 17 & 0 & 3 & 2.57 & 4 \\
Condensation edges & 4 & 17 & 16.79 & 33 & 0 & 4 & 3.70 & 7 \\
DAG depth & 2 & 6 & 6.08 & 14 & 0 & 2 & 2.16 & 4 \\
\bottomrule
\end{tabular}
\end{table}

The two sources cover complementary regimes. LINKS supplies the size range, spanning $6$ to $20$ joints roughly uniformly and reaching solve graphs of depth $14$, but every one of its mechanisms decomposes into dyads alone. Purwar Lab is small and shallow by comparison, capped at $11$ joints and depth $4$, and supplies the structural variety: higher interface arity, $n$-ary bodies, and prismatic pairs. The $153{,}124$ Purwar Lab records with joint code \texttt{PRPR} account for every zero in the Assur-group column: they carry a ground group and no module, and the build marks their topology as incomplete. They are the reason the Purwar Lab minimum for groups, edges, and depth is not the four-bar value.
\subsection{Modules}

\begin{table}[h]
\caption{Module inventory by source. Percentages are shares of that source's
Assur groups (top block) or of its mechanisms (bottom block).}
\label{tab:kb-census-modules}
\centering
\small
\begin{tabular}{lrrrr}
\toprule
& \multicolumn{2}{c}{LINKS} & \multicolumn{2}{c}{Purwar Lab} \\
\cmidrule(lr){2-3}\cmidrule(lr){4-5}
Quantity & count & \% & count & \% \\
\midrule
Groups (total) & 102,534,069 & --- & 10,599,777 & --- \\
Ground groups & 10,020,000 & --- & 2,966,754 & --- \\
Assur groups & 92,514,069 & 100.0 & 7,633,023 & 100.0 \\
\quad order 2 (dyads) & 92,514,069 & 100.0 & 5,761,578 & 75.5 \\
\quad order 3 & 0 & 0.0 & 1,703,554 & 22.3 \\
\quad order 4 & 0 & 0.0 & 167,891 & 2.2 \\
Distinct shape keys & 1 & --- & 33 & --- \\
\midrule
Mechanisms, dyads only & 10,020,000 & 100.0 & 1,081,258 & 36.4 \\
Mechanisms with order $>2$ & 0 & 0.0 & 1,732,372 & 58.4 \\
Mechanisms with no module & 0 & 0.0 & 153,124 & 5.2 \\
\bottomrule
\end{tabular}
\end{table}

\begin{table}[h]
\caption{Per-module distributions. Carried points are tracer points attached
to a module's links and are counted over all groups, including ground.}
\label{tab:kb-census-groups}
\centering
\small
\begin{tabular}{lrrrrrrrr}
\toprule
& \multicolumn{4}{c}{LINKS} & \multicolumn{4}{c}{Purwar Lab} \\
\cmidrule(lr){2-5}\cmidrule(lr){6-9}
Quantity & min & med & mean & max & min & med & mean & max \\
\midrule
Internal joints & 1 & 1 & 1.00 & 1 & 1 & 1 & 1.84 & 7 \\
Boundary joints & 2 & 2 & 2.00 & 2 & 2 & 2 & 2.27 & 4 \\
Bodies & 2 & 2 & 2.00 & 2 & 1 & 2 & 1.99 & 6 \\
Carried points & 0 & 0 & 0.00 & 0 & 0 & 0 & 0.28 & 2 \\
\bottomrule
\end{tabular}
\end{table}

The LINKS module population is a single point: $92{,}514{,}069$ groups, all with the shape key \texttt{c2/i1/e2/L2-2}, one internal joint, two boundary joints, and two binary bars. This is a property of the source, whose mechanisms are generated dyadically, and it is why the audited LINKS-derived records support no claim about higher-order modules. The Purwar Lab population is where the schema's generality is exercised: a quarter of its modules have interface arity above two, modules reach seven internal joints and six bodies, and $33$ distinct shape keys appear. A shape key records order, internal-joint count, boundary-joint count, and the sorted body arities, so \texttt{c2/i3/e2/L3-3} denotes an order-2 module with three internal joints closed by two ternary links.

\begin{table}[h]
\caption{The ten most frequent Purwar Lab shape keys, covering $90.1\%$ of its
Assur groups.}
\label{tab:kb-census-shapes}
\centering
\small
\begin{tabular}{lrr}
\toprule
Shape key & Groups & \% of Assur groups \\
\midrule
\texttt{c2/i1/e2/L2-3} & 1,975,188 & 25.9 \\
\texttt{c2/i2/e2/L2-3} & 1,577,230 & 20.7 \\
\texttt{c3/i1/e3/L3} & 969,112 & 12.7 \\
\texttt{c2/i1/e2/L2-2} & 787,543 & 10.3 \\
\texttt{c2/i3/e2/L3-3} & 538,901 & 7.1 \\
\texttt{c2/i3/e2/L2-4} & 338,516 & 4.4 \\
\texttt{c3/i2/e3/L4} & 231,528 & 3.0 \\
\texttt{c2/i4/e2/L3-4} & 170,360 & 2.2 \\
\texttt{c4/i1/e4/L3} & 146,165 & 1.9 \\
\texttt{c3/i3/e3/L2-2-3-3} & 146,078 & 1.9 \\
\bottomrule
\end{tabular}
\end{table}

The classical dyad, two binary bars on two boundary joints, is only the fourth most common Purwar Lab module at $10.3\%$. The head of the distribution is a binary-plus-ternary pair, and the third entry is a single ternary link on three boundary joints, the order-3 module that a binary joint graph cannot express.

\subsection{Bodies and joints}

\begin{table}[h]
\caption{Typed incidence in the build. Body and joint counts are instances
summed over mechanisms, not distinct types.}
\label{tab:kb-census-incidence}
\centering
\small
\begin{tabular}{lrr}
\toprule
Quantity & LINKS & Purwar Lab \\
\midrule
Bodies & 195,048,138 & 20,698,934 \\
\quad binary & 195,048,138 & 10,321,008 \\
\quad ternary & 0 & 8,938,010 \\
\quad quaternary & 0 & 1,439,916 \\
$n$-ary bodies ($n>2$) & 0 & 10,377,926 \\
Mechanisms with an $n$-ary link & 0 & 2,966,754 \\
Nested bodies & 0 & 614,818 \\
\midrule
Revolute joints & 129,374,453 & 27,423,968 \\
Prismatic joints & 0 & 733,017 \\
Mechanisms with a prismatic joint & 0 & 439,591 \\
\bottomrule
\end{tabular}
\end{table}

Half of all Purwar Lab bodies ($50.1\%$) are ternary or quaternary, and every Purwar Lab mechanism contains at least one such link, so the whole of that source depends on incidence being preserved: flattened to pairwise adjacency, each of these bodies would be indistinguishable from a ring of bars. Prismatic pairs are rarer, $2.6\%$ of joint instances, but they appear in $14.8\%$ of Purwar Lab mechanisms, which is what makes the three slider families ($\texttt{RRRP}$, $\texttt{RRPR}$, $\texttt{PRPR}$) reachable at all.

%% file: tables/corpora_comparison.tex
\begin{table}[t]
\caption{\textbf{The two source corpora store what a mechanism is, and neither stores how it
is built.}  Rows are LINKS~\citep{HeyraniNobari2022Links}, the Purwar Lab four-, six-, and eight-bar corpus~\citep{purwarlab2024dataset}, and the layer derived from them, with source counts as those publications report.  Appendix~\ref{app:kb-census} qualifies the two marked cells.}
\label{tab:corpora-comparison}
\centering
\small
\setlength{\tabcolsep}{5pt}
\begin{tabular}{lrcccccc}
\toprule
& & \multicolumn{4}{c}{Mechanism} & \multicolumn{2}{c}{Construction} \\
\cmidrule(lr){3-6}\cmidrule(lr){7-8}
Resource & Records & Typed & Pose & Coupler & Sliders, & Solve & Group \\
& & incidence & $x_0$ & curve & $n$-ary & order & status \\
\midrule
LINKS                & 10M   & \texttimes & \checkmark & \checkmark$^{\dagger}$  & \texttimes & \texttimes & \texttimes \\
Purwar Lab           & 3M     & \checkmark & \checkmark$^{\ddagger}$ & \checkmark & \checkmark & \texttimes & \texttimes \\
\midrule
Assur-KB (this work) & 12.99M & \checkmark & \checkmark & \checkmark & \checkmark & \checkmark & \checkmark \\
\bottomrule
\end{tabular}
\end{table}